\documentclass{article}
\pdfoutput=1

\usepackage{microtype}
\usepackage{graphicx}
\usepackage{subfigure}
\usepackage{booktabs} 

\usepackage{hyperref}

\usepackage[accepted]{icml2020}

\usepackage{hyperref}    
\usepackage{url}           
\usepackage{booktabs}       
\usepackage{amsfonts}       
\usepackage{nicefrac}       
\usepackage{microtype}
\usepackage{graphicx} 
\usepackage{amssymb}
\usepackage{amsmath}
\usepackage{bbm}
\usepackage{amsfonts}
\usepackage{amsthm}
\usepackage{graphicx}
\usepackage{mathtools}
\usepackage{thmtools, thm-restate}
\usepackage{caption}
\usepackage{supertabular}
\usepackage{datetime}
\usepackage{xcolor}
\usepackage{diagbox}
\usepackage{longtable}
\usepackage{algorithm}
\usepackage{algorithmic}
\usepackage{lipsum}

\usepackage{dsfont}

\usepackage{tikz}
\usetikzlibrary{shapes,backgrounds}

\usepackage[normalem]{ulem}
\usepackage{stmaryrd}

\usepackage{thmtools, thm-restate}

\theoremstyle{definition}
\declaretheorem{definition}
\theoremstyle{plain}

\declaretheorem{theorem}
\declaretheorem{corollary}

\DeclareMathOperator*{\argmin}{argmin}
\newcommand{\good}{G} 

\newcommand{\Tau}{\mathrm{T}}
\newcommand{\loss}{\ell}

\usepackage{xspace}
\makeatletter
\DeclareRobustCommand\onedot{\futurelet\@let@token\@onedot}
\def\@onedot{\ifx\@let@token.\else.\null\fi\xspace}

\def\iid{{i.i.d}\onedot}
\def\eg{{e.g}\onedot} 
\def\ie{{i.e}\onedot}

\makeatother

\newcommand{\chl}[1]{\textcolor{red}{\textbf{CHL:}{#1}}}
\newcommand{\niko}[1]{\textcolor{blue}{\textbf{NK:}{#1}}}

\allowdisplaybreaks

\newcommand{\robust}{robust\xspace}
\newcommand{\ttimes}{\!\times\!}

\newcommand{\myparagraph}[1]{\noindent\textbf{#1}}

\icmltitlerunning{Adversarial Multi-Source PAC Learning}

\begin{document}

\twocolumn[
\icmltitle{On the Sample Complexity of Adversarial Multi-Source PAC Learning}



\icmlsetsymbol{equal}{*}

\begin{icmlauthorlist}
\icmlauthor{Nikola Konstantinov}{ist}
\icmlauthor{Elias Frantar}{ist,tu}
\icmlauthor{Dan Alistarh}{ist}
\icmlauthor{Christoph H. Lampert}{ist}
\end{icmlauthorlist}

\icmlaffiliation{ist}{Institute of Science and Technology Austria, Klosterneuburg, Austria}
\icmlaffiliation{tu}{Vienna University of Technology, Vienna, Austria}

\icmlcorrespondingauthor{Nikola Konstantinov}{nkonstan@ist.ac.at}

\icmlkeywords{Robustness, Distributed Learning, Statistical Learning Theory, PAC}

\vskip 0.3in
]



\printAffiliationsAndNotice{}  

\begin{abstract}

We study the problem of learning from multiple untrusted data sources, 
a scenario of increasing practical relevance given the recent emergence 
of crowdsourcing and collaborative learning paradigms. Specifically, we 
analyze the situation in which a learning system obtains datasets from 
multiple sources, some of which might be biased or even adversarially 
perturbed. It is known that in the single-source case, an adversary 
with the power to corrupt a fixed fraction of the training data can 
prevent PAC-learnability, that is, even in the limit of infinitely much 
training data, no learning system can approach the optimal test error. 
In this work we show that, surprisingly, the same is not true in the 
multi-source setting, where the adversary can arbitrarily corrupt a 
fixed fraction of the data sources. Our main results are a 
generalization bound that provides finite-sample guarantees for this 
learning setting, as well as corresponding lower bounds. Besides 
establishing PAC-learnability our results also show that in a 
cooperative learning setting sharing data with other parties has 
provable benefits, even if some participants are malicious.

\end{abstract}

\section{Introduction}
\label{sec:introduction}
An important problem of current machine learning research 
is to make learned systems more \emph{trustworthy}. 
One particular aspect of this is \emph{robustness} 
against data of unexpected or even adversarial nature. 
Robustness at prediction time has recently received 
a lot of attention, in particular with work on the
detection of \emph{out-of-distribution} conditions~\cite{hendrycks17iclr,liang2018enhancing,lee2018simple} and  
protection against \emph{adversarial examples}~\cite{raghunathan2018certified,singh2018fast,cohen2019certified}. 
Robustness at training time, however, is represented less 
prominently, despite also being of great importance. 
One reason might be that learning from a potentially 
adversarial data source is very hard: a classic result 
states that when a fixed fraction of the training dataset 
is adversarially corrupted, successful learning in 
the PAC sense is not possible anymore ~\cite{kearns1993learning}. 
In other words, there exists no \emph{\robust learning algorithm} 
that could overcome the effects of adversarial corruptions in a constant fraction of the training dataset and approach the optimal model, even in 
the limit of infinite data. 

In this work, we study the question of robust learning in the 
multi-source case, \ie when more than one dataset is available 
for training.
This is a situation of increasing relevance in the era of 
\emph{big data}, where machine learning models tend to be 
trained on very large datasets. To create these, one commonly 
relies on distributing the task of collecting and annotating 
data, \eg to crowdsourcing~\cite{sheng2019machine} services, 
or by adopting a collective or federated learning scenario~\cite{mcmahan2017}. 

Unfortunately, relying on data from other parties comes with the 
danger that some of the sources might produce data of lower quality 
than desired, be it due to negligence, bias or malicious behaviour. 
Consequently, the analogous question to the classic problem described 
above is the following, which we refer to as \emph{adversarial 
multi-source learning}. \emph{Given a number of \iid datasets, 
a constant fraction of which might have been adversarially 
manipulated, is there a learning algorithm that overcomes the effect 
of the corruptions and approaches an optimal model?}

In this work, we study this problem formally and provide a positive 
answer. 
Specifically, \emph{our main result is an upper bound on the sample 
complexity of adversarial multi-source learning, that holds as long as less 
than half of sources are manipulated (Theorem~\ref{thm:upper_bound}).}

A number of interesting results follow as immediate corollaries. 
First, we show that any hypothesis class that is uniformly convergent and hence PAC-learnable 
in the classical \iid sense is also PAC-learnable in the adversarial 
multi-source scenario. This is in stark contrast to the single-source 
situation where, as mentioned above, no non-trivial hypothesis 
class is robustly PAC-learnable.
As a second consequence, we obtain the insight that in a cooperative 
learning scenario, every honest party can benefit from sharing their 
data with others, as compared to using their own data only, even if some of the participants are malicious. 

Besides our main result we prove two additional theorems that shed
light on the difficulty of adversarial multi-source learning. 
First, we prove that the na\"ive but common strategy of simply 
merging all data sources and training with some robust procedure on 
the joint dataset cannot result in a robust learning algorithm 
(Theorem~\ref{thm:lower_bound_single_source_learner}). 
Second, we prove a lower bound on the sample complexity under 
very weak conditions (Theorem~\ref{thm:no_free_lunch}).
This result shows that under adversarial conditions a slowdown of 
convergence is unavoidable, and that in order to approach optimal 
performance, the number of samples per source must necessarily 
grow, while increasing the number of sources need not help.

\section{Related work}
\label{sec:related_work}
To our knowledge, our results are the first that formally 
characterize the statistical hardness of learning from 
multiple \iid sources, when a constant fraction of them 
might be adversarially corrupted.
There are a number of conceptually related works, though,
which we will discuss for the rest of this section.

\citet{qiao2018learning}, as well as the follow-up works of \citet{chen2019efficiently, jain2019robust}, aim at estimating discrete 
distributions from multiple batches of data, some of 
which have been adversarially corrupted.
The main difference to our results is the focus on 
finite data domains and estimating the underlying 
probability distribution rather than learning a 
hypothesis.

\citet{qiao2018outliers} studies collaborative binary 
classification: a learning system has access to 
multiple training datasets and a subset of them 
can be adversarially corrupted. In this setup, the uncorrupted sources are allowed
to have different input distributions, but share a common labelling function.
The author proves that it is possible to robustly 
learn individual hypotheses for each source, but 
a single shared hypothesis cannot be learned robustly.
For the specific case that all data distributions 
are identical, the setup matches ours, though only 
for binary classification in the realizable case, 
and with a different adversarial model. 

In a similar setting, \citet{pmlr-v97-mahloujifar19a} 
show, in particular, that an adversary can increase 
the probability of any "bad property" of the learned
hypothesis by a term at least proportional to the 
fraction of manipulated sources.
These results differ from ours, by their assumption 
that different sources have different distributions,
which renders the learning problem much harder.

In \citet{konstantinov2019robust}, a learning system has access 
to multiple datasets, some of which are manipulated, and the
authors prove a generalization bound and propose an algorithm 
based on learning with a weighted combination of all datasets.
The main difference to our work is that their proposed method crucially relies on a trusted subset of the data being known to the learner. Their adversary is also weaker, as it cannot influence the data points directly, but only change the distribution from which they are sampled, and the work does not provide finite sample guarantees.

There are a number of classic results on the fundamental limits of PAC learning from a \textit{single labelled set of samples}, a fraction of which can be arbitrarily corrupted, \eg \cite{kearns1993learning, bshouty2002pac}. We compare our results against this classic scenario in Section \ref{sec:main_result}. 

Another related general direction is the research on Byzantine-resilient distributed learning, which has seen significant interest recently, e.g.~\cite{blanchard2017machine, chen2017distributed, yin2018byzantine, yin2019defending, NIPS2018_7712}. There the focus is on learning by exchanging gradient updates between nodes in a distributed system, an unknown fraction of which might be corrupted by an omniscient adversary and may behave arbitrarily. These works tend to design defences for specific gradient-based optimization algorithms, such as SGD, and their theoretical analysis usually assumes strict conditions on the objective function, such as convexity or smoothness. Nevertheless, the (nearly) tight sample complexity upper and lower bounds developed for Byzantine-resilient gradient descent~\cite{yin2018byzantine} and its stochastic variant~\cite{NIPS2018_7712} are relevant to our results and are therefore discussed in detail in Sections~\ref{sec:rates_of_convergence} and~\ref{sec:hardness_of_multisource_learning}. 

The work of \citet{awasthi2017efficient} considers learning from crowdsourced data, where some of the workers might behave arbitrarily. However, they only focus on label corruptions. \citet{feng2017fundamental} consider the fundamental limits of learning from adversarial distributed data, but in the case when \textit{each of the nodes} can iteratively send corrupted updates with certain probability. \citet{feng2014distributed} provide a method for distributing the computation of any robust learning algorithm that operates on a single large dataset. There is also a large body of literature on attacks and defences for federated learning, \eg \cite{pmlr-v97-bhagoji19a, fung2018mitigating}. Apart from focusing on iterative gradient-based optimization procedures, these works also allow for natural variability in the distributions of the uncorrupted data sources.

\section{Preliminaries}
\label{sec:preliminaries}
In this section we introduce the technical definitions
that are 
necessary to formulate and prove our main results. 
We start by reminding the reader of the classical 
notion of PAC-learnability and uniform convergence,
as they can be found in most machine learning textbooks. 
We then introduce the setting of learning from multiple 
sources and notions of adversaries of different strengths. 

\subsection{Notation and Background}
Let $\mathcal{X}$ and $\mathcal{Y}$ be given input and 
output sets, respectively,
and $\mathcal{D} \in \mathcal{P}(\mathcal{X}\ttimes\mathcal{Y})$ be
a fixed but unknown probability distribution.
By $\loss:\mathcal{Y}\ttimes\mathcal{Y} \rightarrow \mathbb{R}$ 
we denote a loss function, and by 
$\mathcal{H}\subset \{h: \mathcal{X}\rightarrow \mathcal{Y}\}$
a set of hypotheses. All of these quantities are assumed arbitrary 
but fixed for the purpose of this work. 

A \emph{(statistical) learner} is a function 
$\mathcal{L}:\cup_{m=1}^{\infty}(\mathcal{X}\ttimes\mathcal{Y})^m \rightarrow \mathcal{H}$.
In the classic supervised learning scenario, the learner has access to 
a \emph{training set} of $m$ labelled examples, $\{(x_1, y_1), \ldots, (x_m, y_m)\}$, 
sampled \iid from $D$, and aims at learning a hypothesis $h\in\mathcal{H}$ 
with small \emph{risk}, \ie expected loss, under the unknown data 
distribution,
\begin{equation}\label{eqn:expected_loss}
\mathcal{R}(h) = \mathbb{E}_{(x,y)\sim \mathcal{D}}(\loss(h(x), y)).
\end{equation} 
\emph{PAC-learnability} is a key property of the hypothesis set, 
which ensures the existence of an algorithm that performs 
successful learning:
\begin{definition}[\textbf{PAC-Learnability}]
\label{defn:pac_learnability}
We call $\mathcal{H}$ (agnostic) \emph{probably approximately correct (PAC)
learnable} 
with respect to $\loss$, if there exists a learner $\mathcal{L}$ and a function $m_{\mathcal{H},\loss}: (0,1)\ttimes(0,1) \rightarrow \mathbb{N}$, such that for any $\epsilon, \delta \in (0,1)$, whenever $S$ is a set of $m \geq m_{\mathcal{H},\loss}(\epsilon, \delta)$ \iid labelled samples from $\mathcal{D}$, then with probability at least $1-\delta$ over the sampling of $S$:
\begin{equation}
\mathcal{R}(\mathcal{L}(S)) \leq \min_{h\in\mathcal{H}}\mathcal{R}(h) + \epsilon.
\end{equation}
\end{definition}

Another important concept related to PAC-learnability is that of \emph{uniform convergence}.
\begin{definition}[\textbf{Uniform convergence}]
\label{defn:uniform_convergence}
We say that $\mathcal{H}$ has the \emph{uniform convergence} 
property with respect to $\loss$ with rate $s_{\mathcal{H}, \loss}$,
if there exists a function $s_{\mathcal{H}, \loss}: 
\mathbb{N}\ttimes(0,1)\ttimes\bigcup_{m=1}^{\infty} \left(\mathcal{X}\ttimes\mathcal{Y}\right)^m
\rightarrow \mathbb{R}$, such that for any distribution $\mathcal{D}\in\mathcal{P}\left(X\ttimes Y\right)$ 
and any $\delta \in (0,1)$:
\begin{itemize}
\item given $m$ samples $S = \{\left(x_1, y_1\right), \ldots, \left(x_m, y_m\right)\}\stackrel{\iid}{\sim}\mathcal{D}$, 
with probability at least $1 - \delta$ over the data :
\begin{equation}
\label{eqn:concentration_property}
\sup_{h\in\mathcal{H}} |\mathcal{R}(h) - \widehat{\mathcal{R}}(h)| \leq s_{\mathcal{H}, \loss}\left(m, \delta, S\right),
\end{equation}
where $\hat{\mathcal{R}}(h)$ is the empirical risk of the hypothesis $h$.
\item $s_{\mathcal{H}, \loss}\left(m, \delta, S_m\right)\rightarrow 0$ as $m\rightarrow \infty$, 
for any sequence $(S_m)_{m\in\mathbb{N}}$ with $S_m\in(\mathcal{X}\ttimes \mathcal{Y})^m$. 
\end{itemize}
\end{definition}
Throughout the paper we drop the dependence on $\mathcal{H}$ and $\loss$ and simply write $s$ for $s_{\mathcal{H}, \loss}$. 
Note that above definition is equivalent to the classic definition 
of uniform convergence (\eg Chapter 4 in \cite{shalev2014understanding}).
We only introduce an explicit notation, $s$, for the sample complexity 
rate of uniform convergence, as this simplifies the layout of our analysis
later. 
It is well-known that uniform convergence implies PAC-learnability 
and that the opposite is also true for agnostic binary classification \cite{shalev2014understanding}.

\subsection{Multi-source learning} 
Our focus in this paper is on learning from multiple data sources. 
For simplicity of exposition, we assume that they all provide
the same number of data points, \ie the training data consists 
of $N$ groups of $m$ samples each, where $m, N \in \mathbb{N}$ 
are fixed integers. 

Formally, we denote by 
$\left(\mathcal{X}\times\mathcal{Y}\right)^{N\times m}$ the set 
of all possible collections (\ie unordered sequences) of $N$ 
groups of $m$ datapoints each. 
A (statistical) \emph{multi-source learner} is a function 
$\mathcal{L}: \cup_{N=1}^{\infty} \cup_{m=1}^{\infty} \left(\mathcal{X}\ttimes\mathcal{Y}\right)^{N\times m} \rightarrow \mathcal{H}$ 
that takes such a collection of datasets and returns a 
predictor from $\mathcal{H}$.

\subsection{Robust Multi-Source Learning}

Informally, one considers a learning system \emph{robust} 
if it is able to learn a good hypothesis, even when the 
training data is not perfectly \iid , but contains 
some artifacts, \eg annotation errors, a selection bias 
or even malicious manipulations. 
Formally, one models this by assuming the presence of an 
\emph{adversary}, that observes the original datasets and outputs potentially manipulated versions.
The learner then has to operate on the manipulated data 
without knowledge of what the original one had been or what
manipulations have been made.

\begin{definition}[\bf Adversary]
An adversary is any function $\mathcal{A}:\left(\mathcal{X}\ttimes\mathcal{Y}\right)^{N\times m} \rightarrow\left(\mathcal{X}\ttimes\mathcal{Y}\right)^{N\times m}$.
\end{definition}

Throughout the paper, we denote by $S' = \{S'_1, S'_2, \ldots, S'_N\}$ the original, uncorrupted datasets, drawn \iid from $\mathcal{D}$, and by $S = \{S_1, S_2, \ldots, S_N\} = \mathcal{A}(S')$ the datasets returned by the adversary.

Different scenarios are obtained by giving the adversary different 
amounts of power. 
For example, a weak adversary might only be able to randomly flip 
labels, \ie simulate the presence of label noise.
A much stronger adversary would be one that can potentially 
manipulate all data and do so with knowledge not only of all of the
datasets but also of the underlying data distribution and the 
learning algorithm to be used later.

In this work, we adopt the latter view, as it leads to much 
stronger robustness guarantees. 
We define two adversary types that can make arbitrary 
manipulations to data sources, but only influence a 
certain subset of them. 

\begin{definition}[\bf Fixed-Set Adversary]\label{defn:fixed-set_adversary}
Let $\good \subset [N]$. An adversary is called \emph{fixed-set (with preserved set $\good$)}, 
if it only influences the datasets outside of $G$. That is, $S_i = S^{'}_i$ for all $i \in \good$.
\end{definition}
\begin{definition}[\bf Flexible-Set Adversary]\label{defn:flexibke-set_adversary}
Let $k\in\{0,1, \ldots, N\}$. An adversary is called \emph{flexible-set} (with preserved size $k$), 
if it can influence any $N - k$ of the $N$ given datasets. That is, there exists a set $\good \subset [N]$, such that $|\good| = k$ and $S_i = S'_i$ for all $i\in \good$.
\end{definition}
In both cases, we call the fraction $\alpha$ of corrupted datasets the 
\emph{power} of the adversary, \ie $\alpha=\frac{N-|\good|}{N}$
for the fixed-set and $\alpha=\frac{N-k}{N}$ for the flexible-set
adversaries.

While similarly defined, the fixed-set adversary is strictly weaker
than the flexible-set one, as the latter one can first inspect 
all data and then choose which subset to modify, while the former 
one is restricted to a fixed, data-independent subset of sources. In particular, the flexible-set adversary can already bias the distribution of the data by throwing out a carefully chosen set of sources, before replacing them with new data.

Both adversary models are inspired by real-world considerations and analogs 
have appeared in a number of other research areas.
The fixed-set adversaries can model a situation in which $N$ parties 
collaborate on a single learning task, but an unknown and fixed set 
of them are compromised, \eg by hackers, that can act maliciously 
and collude with each other. 
This is a similar reasoning as in \emph{Byzantine-robust optimization}, 
where an unknown subset of computing nodes are assumed to behave 
arbitrarily, thereby disrupting the optimization progress.

The second adversary corresponds to a situation where a malicious 
party can observe all of the available datasets and choose which ones 
to corrupt, up to a certain budget. 
This is similar to classic models in the fields of robust PAC learning, 
\eg \cite{bshouty2002pac}, and robust mean estimation, \eg 
\cite{diakonikolas2019robust}, where the adversary itself can 
influence which subset of the data to modify once the whole dataset 
is observed.

Whether robust learning in the presence of an adversary is possible
for a certain hypothesis set or not is captured by the following 
definition:
\begin{definition}\label{defn:multi_source_pac}
A hypothesis set, $\mathcal{H}$, is called \emph{multi-source 
PAC-learnable} against the class of fixed-set adversaries (or flexible-set adversaries) and with respect to $\loss$, if there exists a 
multi-source learner $\mathcal{L}$ and a function $m: (0,1)^2 \rightarrow \mathbb{N}$, 
such that for any $\epsilon, \delta \in (0,1)$ and any set $\good \subset [N]$ of size $|\good| > \frac{1}{2}N$ (or any $\alpha < \frac{1}{2}$), whenever $S^{'} \in \left(\mathcal{X}\ttimes\mathcal{Y}\right)^{N\times m}$ is a collection of $N$ datasets of $m \geq m(\epsilon, \delta)$ \iid labelled samples from $\mathcal{D}$ each, then with probability at least $1-\delta$ over the sampling of $S^{'}$:
\begin{align}
\mathcal{R}(\mathcal{L}(\mathcal{A}(S^{'})) \leq \min_{h\in\mathcal{H}}\mathcal{R}(h) + \epsilon,
\end{align}
uniformly against all fixed-set adversaries with preserved set $\good$ (or all flexible-set adversaries of power $\alpha$). A learner, $\mathcal{L}$, with this property is called \emph{robust multi-source learner} for $\mathcal{H}$.
\end{definition}

In particular, the same learner $\mathcal{L}$ should work against any adversary and for any $\alpha$ or set $\good$. In the same time, the adversary is arbitrary once $\mathcal{L}$ is fixed, so in particular it can depend on the learning algorithm.

Note that the robust learner should achieve optimal error as 
$m\rightarrow \infty$, while $N$ can stay constant. This reflects
that we want to study adversarial multi-source learning in the 
context of a constant and potentially not very large number of 
sources. 
In fact, our lower bound results in Section \ref{sec:lower_bounds} 
show that the adversary can always prevent the learner from approaching 
optimal risk in the opposite regime of constant $m$ and $N\to\infty$.

\section{Sample Complexity of Robust Multi-Source Learning}
\label{sec:upper_bounds}
In this section, we present our main result, a theorem that 
states that whenever $\mathcal{H}$ has the uniform convergence 
property, there exists an algorithm that guarantees a bounded 
excess risk against both the fixed-set and the flexible-set adversary.
We then derive and discuss some instantiations of the general 
result that shed light on the sample complexity of PAC 
learning in the adversarial multi-source learning setting.
Finally, we provide a high-level sketch of the theorem's 
proof.

\subsection{Main result}\label{sec:main_result}

\begin{theorem}
\label{thm:upper_bound}
Let $N, m, k \in \mathbb{N}$ be integers, such that $k \in (N/2, N]$. 
Let $\alpha = \frac{N-k}{N} < \frac{1}{2}$ be the proportion of corrupted sources. 
Assume that $\mathcal{H}$ has the uniform convergence property with 
rate function $s$. 
Then there exists a learner 
$\mathcal{L}: \left(\mathcal{X}\ttimes\mathcal{Y}\right)^{N \times m} \rightarrow \mathcal{H}$ 
with the following two properties.
\begin{itemize}
\item[(a)] Let $\good$ be a fixed subset of $[N]$ of size $|\good| = k$. 
For $S^{'} = \{S^{'}_1, \ldots, S^{'}_N\}\stackrel{\iid}{\sim}\mathcal{D}$, with probability at least $1-\delta$ 
over the sampling of $S'$:
\begin{align}
\label{eqn:no_trusted_data_fixed-set}
& \mathcal{R}(\mathcal{L}(\mathcal{A}(S^{'}))) - \min_{h\in\mathcal{H}}\mathcal{R}(h)
  \\ & \leq 2s\big(km, \frac{\delta}{2}, S_{\good}\big) + 6\alpha\max_{i \in [N]}s\big(m, \frac{\delta}{2N}, S_{i}\big), \nonumber
\end{align}
uniformly against all fixed-set adversaries with preserved set $\good$, where $S = \{S_1, \ldots, S_N\} = \mathcal{A}(S^{'})$ is
the dataset modified the adversary and $S_{\good} = \cup_{i\in\good}S_i$ is the set of all uncorrupted data.
\item[(b)] 
For $S^{'} = \{S^{'}_1, \ldots, S^{'}_N\}\stackrel{\iid}{\sim}\mathcal{D}$, with probability at least $1-\delta$ over the sampling of $S'$:
\begin{align}
\label{eqn:no_trusted_data_flexible-set}
&\mathcal{R}(\mathcal{L}(\mathcal{A}(S^{'}))) - \min_{h\in\mathcal{H}}\mathcal{R}(h)
\\ & \leq 2s\big(km, \frac{\delta}{2 \binom{N}{k}}, S_{\good}\big) + 6 \alpha \max_{i\in [N]} s\big(m, \frac{\delta}{2N}, S_i\big), \nonumber
\end{align}
uniformly against all flexible-set adversaries with preserved size $k$, where $S = \{S_1, \ldots, S_N\} = \mathcal{A}(S^{'})$ is the dataset returned by the adversary, $\good$ is the set of sources not modified by the adversary and $S_{\good} = \cup_{i\in\good}S_i$ is the set of all uncorrupted data.

\end{itemize}
\end{theorem}
The learner $\mathcal{L}$ is in fact explicit, we define and 
discuss it in the proof sketch that we provide in Section 
\ref{sec:proof_of_thm}. The complete proof is provided in 
the supplementary material.

As an immediate consequence we obtain:
\begin{corollary}\label{cor:pac_learnability}
Assume that $\mathcal{H}$ has the uniform convergence property.
Then $\mathcal{H}$ is multi-source PAC-learnable against the class of 
fixed-set and the class of flexible-set adversaries. 
\end{corollary}
\begin{proof}
It suffices to show that for any $\delta\in(0,1)$, the right hand sides 
of \eqref{eqn:no_trusted_data_fixed-set} and 
\eqref{eqn:no_trusted_data_flexible-set}
converge to $0$ for $m\to\infty$. 
This it true, since $s(\bar m, \bar\delta, \bar S) \rightarrow 0$ 
as $\bar m \rightarrow \infty$ for any $\bar\delta$ and $\bar S$, 
by the definition of uniform convergence. Since the same learner works regardless of the choice of $\good$ and/or $\alpha$, the result follows.
\end{proof}

\noindent\textbf{Discussion.}\ 
Corollary~\ref{cor:pac_learnability} is in sharp contrast with 
the situation of single dataset PAC robustness. 
In particular, \citet{bshouty2002pac} study a setup where an
adversary can manipulate a fraction $\alpha$ datapoints out 
of a dataset with $m$ \iid-sampled elements\footnote{To be 
precise, the number of influenced points has to be binomially 
distributed with mean $\alpha m$, but the difference between
this and the deterministic setting becomes irrelevant for 
$m\to\infty$.}. 
The authors show 
that in the binary realizable case, for any hypothesis 
space with at least two functions, no learning algorithm 
can learn a hypothesis with risk less than $2\alpha$ with  probability greater than $1/2$.
Similarly, \citet{kearns1993learning} showed that for an 
adversary that modifies each data point with constant 
probability $\alpha$, no algorithm can learn a hypothesis 
with accuracy better than $\alpha/(1-\alpha)$.
Both results hold regardless of the value of $m$, thus 
showing that PAC-learnability is not fulfilled.

\subsection{Rates of convergence}\label{sec:rates_of_convergence}
While Theorem~\ref{thm:upper_bound} is most general, it does 
not yet provide much insight into the actual sample complexity
of the adversarial multi-source PAC learning problem, because 
the rate function $s$ might behave in different ways. 
In this section we give more explicit upper bounds in terms a 
standard complexity measure of hypothesis spaces -- the Rademacher 
complexity. 
Let 
\begin{align}
\mathfrak{R}_S\left(\loss \circ \mathcal{H}\right) = \mathbb{E}_{\sigma}\Big(\sup_{h\in\mathcal{H}} \frac{1}{n} \sum_{i=1}^{n}\sigma_i \loss(h(x_i), y_i)\Big),
\end{align}
be the (empirical) Rademacher complexity of $\mathcal{H}$ with respect to 
the loss function $\loss$ on a sample $S = \{(x_1, y_1), \ldots, (x_n, y_n)\}$. Here $\{\sigma_i\}_{i=1}^n$ are \iid Rademacher random variables.
Let $S_{\good} = \bigcup_{i\in\good}S_i$, $\mathfrak{R}_i = \mathfrak{R}_{S_i}(\loss \circ \mathcal{H})$ and $\mathfrak{R}_{\good} = \mathfrak{R}_{S_{\good}}(\loss \circ \mathcal{H})$. 

\subsubsection{Rates for the fixed-set adversary.} 
An application of Theorem \ref{thm:upper_bound} with a standard uniform concentration result gives:
\begin{corollary}
\label{cor:rademacher_rates1}
In the setup of Theorem \ref{thm:upper_bound},
against any fixed-set adversary, it holds that 
\begin{align}\label{eqn:fixed-set_rademacher}
\mathcal{R}(\mathcal{L}(\mathcal{A}(S^{'}))) &- \min_{h\in\mathcal{H}}\mathcal{R}(h)  
\leq 4\mathfrak{R}_{\good} + 6\sqrt{\frac{\log(\frac{4}{\delta})}{2km}} 
\\ + &\alpha\Big(18\sqrt{\frac{\log\left(\frac{4N}{\delta}\right)}{2m}} + 12\max_{i\in [N]}\mathfrak{R}_i\Big).   \nonumber
\end{align}
\end{corollary}

The full proof is included in the supplementary material.

In many common learning settings, the Rademacher complexity scales 
as $\mathcal{O}(1/\sqrt{n})$ with the sample size $n$ 
(see \eg \cite{bousquet2004introduction}). Thereby, we obtain 
the following rates against the fixed-set adversary:
\begin{equation}\label{eqn:convergence_rates_ours}
\widetilde{\mathcal{O}}\Big(\frac{1}{\sqrt{km}} + \alpha\frac{1}{\sqrt{m}}\Big),
\end{equation}
where the $\widetilde{\mathcal{O}}$-notation hides constant and logarithmic factors. 

The results in Corollary \ref{cor:rademacher_rates1} and Equation (\ref{eqn:convergence_rates_ours}) allow us to reason about the type of guarantees that can be achieved given a certain amount of data. However, they also imply an explicit upper bound on the sample complexity of adversarial multi-source learning (i.e. an upper bound on the smallest possible $m(\epsilon, \delta)$ in Definition \ref{defn:multi_source_pac}) of the form:
\begin{equation}
\label{eqn:sample_complexity_upper_bound1}
m(\epsilon, \delta) \leq \mathcal{O}\left(\frac{\log(\frac{N}{\delta})}{\epsilon^2}\left(\frac{1}{\sqrt{(1-\alpha)N}} + \alpha\right)^2\right).
\end{equation}

\myparagraph{Discussion.}
We can make a number of observations from Equation~\eqref{eqn:convergence_rates_ours}.
The $\sqrt{1/km}$-term is the rate one expects when learning 
from $k$ (uncorrupted) sources of $m$ samples each, that is from all the available uncorrupted data. The 
$\sqrt{1/m}$-term reflects the rate when learning from 
any single source of $m$ samples, \ie without the benefit of sharing
information between sources. The latter enters weighted by $\alpha$,
\ie it is directly proportional to the power of the adversary.
In the limit of $\alpha \rightarrow 0$ (\ie all $N$ sources are 
uncorrupted, $k\to N$),
the bound becomes 
$\widetilde{\mathcal{O}}(\sqrt{1/Nm})$. 
Thus, we recover the classic convergence rate for 
learning from $Nm$ samples in the non-realizable case. 
This fact is interesting, as the robust learner of 
Theorem~\ref{thm:upper_bound} actually does not need
to know the value of $\alpha$ for its operation.
Consequently, the same algorithm will work robustly 
if the data contains manipulations but without an
unnecessary overhead (\ie with optimal rate), 
if all data sources are in fact uncorrupted. 

Another insight follows from the fact that for reasonably 
small $\alpha$, we have:
\begin{align}
\widetilde{\mathcal{O}}\Big(\frac{1}{\sqrt{km}} + \alpha\frac{1}{\sqrt{m}}\Big) \ll \widetilde{\mathcal{O}}\Big(\frac{1}{\sqrt{m}}\Big),
\end{align}
so learning from multiple, even potentially manipulated, datasets 
converges to a good hypothesis faster than learning from a 
single uncorrupted dataset. 
This fact can be interpreted as encouraging cooperation: any of 
the \emph{honest} parties in the multi-source setting with 
fixed-set adversary will benefit from making their data available
for multi-source learning, even if some of the other parties 
are malicious. 

\myparagraph{Comparison to Byzantine-robust optimization.} 
Our obtained rates for the fixed-set adversary can also be 
compared to the state-of-art convergence results for 
Byzantine-robust distributed optimization, where the compromised 
nodes are also fixed, but unknown. 
\citet{yin2018byzantine} and \citet{NIPS2018_7712} develop robust 
algorithms for gradient descent and stochastic gradient descent 
respectively, achieving convergence rates of order
\begin{equation}\label{eqn:convergence_rates_byzantine}
\widetilde{\mathcal{O}}\Big(\frac{1}{\sqrt{km}} + \alpha\frac{1}{\sqrt{m}} + \frac{1}{m}\Big)
\end{equation}
for $\alpha < 1/2$  unknown.
Clearly, these rates resemble ours, except for the additional 
$1/m$-term, which matters when $\alpha$ is $0$ or very small.
As shown in \citet{yin2018byzantine}, this term can also be made to
disappear if an upper bound $\beta \geq \alpha$ is assumed to be known a priori. 

Overall, these similarities should not be over-interpreted, 
as the results for Byzantine-robust optimization describe practical 
gradient-based  algorithms for distributed 
optimization under various technical assumptions, such as convexity, 
smoothness of the loss function and bounded variance of the gradients. 
In contrast, our work is purely statistical, not taking computational
cost into account, but holds in a much broader context, for any 
hypothesis space that has the uniform convergence property of
suitable rate and without constraints on the optimization method to 
be used. 
Additionally, our rates improve automatically in situations 
where uniform convergence is faster. 

\subsubsection{Rates for the flexible-set adversary} 
An analogous result to Corollary~\ref{cor:rademacher_rates1} holds
also for flexible-set adversaries:

\begin{corollary}
\label{cor:rademacher_rates2}
In the setup of Theorem \ref{thm:upper_bound}, against any flexible-set 
adversary, it holds that 
\begin{align}\label{eqn:malicious_rademacher}
&\mathcal{R}(\mathcal{L}(\mathcal{A}(S^{'}))) - \min_{h\in\mathcal{H}}\mathcal{R}(h)  
\\ 
& \leq 4\mathfrak{R}_{\good} + 12\alpha \max_{i\in [N]}\mathfrak{R}_i + \widetilde{\mathcal{O}}\left(\frac{\sqrt[4]{\alpha}}{\sqrt{m}}\right). \nonumber
\end{align}
\end{corollary}
The proof is provided in the supplemental material.

Making the same assumptions as above, we obtain a sample 
complexity rate 
\begin{equation}\label{eqn:convergence_rates_ours2}
\widetilde{\mathcal{O}}\left(\frac{1}{\sqrt{km}} + \frac{\sqrt[4]{\alpha}}{\sqrt{m}}\right).
\end{equation}
which differs from \eqref{eqn:convergence_rates_ours} 
only in the rate of dependence on $\alpha$\footnote{In fact, we believe the $\sqrt[4]{\alpha}$-term 
to be an artifact of our proof technique, but currently do not have a bound with improved dependence on $\alpha$.}, 
which, if at all, matters only for very small (but non-zero) $\alpha$. 
Despite the difference, most of our discussion above still 
applies. In particular, even for the flexible-set adversary
the same learning algorithm exhibits robustness for $\alpha>0$
and achieves optimal rates for $\alpha=0$. 

Moreover, an explicit upper bound on the sample complexity against a flexible-set adversary is given by:
\begin{equation}
\label{eqn:sample_complexity_upper_bound2}
m(\epsilon, \delta) \leq \widetilde{\mathcal{O}}\left(\frac{1}{\epsilon^2}\left(\frac{1}{\sqrt{(1-\alpha)N}} + \sqrt[4]{\alpha}\right)^2\right).
\end{equation}
\subsection{Proof Sketch for Theorem \ref{thm:upper_bound}}\label{sec:proof_of_thm}
The proof of Theorem \ref{thm:upper_bound} consists of two parts.
First, we introduce a filtering algorithm, that attempts to determine 
which of the data sources can be trusted, meaning that 
it should be safe to use them for training a hypothesis. 
Note that this can be because they were not manipulated, or because 
the manipulations are too small to have negative consequences. 
The output of the algorithm is a new \emph{filtered} training 
set, consisting of all data from the trusted sources only. 
Second, we show that training a standard single-source learner 
on the filtered training set yields the desired results. 

\begin{algorithm}[t]
\caption{} \label{alg:filtering}
\begin{algorithmic}
\INPUT Datasets $S_1, \ldots, S_N$
\STATE  Initialize $\Tau = \{\}$ \qquad // trusted sources
\FOR{$i=1,\dots,N$}
\IF{$d_{\mathcal{H}}\big(S_i, S_j\big) \leq s \left(m, \frac{\delta}{2N}, S_i\right) + s \left(m, \frac{\delta}{2N}, S_j\right),$\\
\qquad for at least $\lfloor \frac{N}{2} \rfloor$ values of $j\not = i$,}
\STATE{$\Tau = \Tau \cup \{i\}$}  
\ENDIF
\ENDFOR
\OUTPUT $\bigcup_{i\in \Tau} S_i$ \qquad // all data of trusted sources
\end{algorithmic}
\end{algorithm}

\noindent{\textbf{Step 1.}}
Pseudo-code for the filtering algorithm is provided 
in Algorithm~\ref{alg:filtering}.
The crucial component is a carefully chosen notion of 
distance between the datasets, called \emph{discrepancy}, 
that we define and discuss below. 
It guarantees that if two sources are close to each other
then the difference of training on one of them compared to 
the other is small. 

To identify the trusted sources, the algorithm checks for 
each source how close it is to all other sources with 
respect to the discrepancy distance. 
If it finds the source to be closer than a threshold to at 
least half of the other sources, it is marked as trusted, 
otherwise it is not. 
To show that this procedure does what it is intended to do
it suffices to show that two properties hold with high probability:
1) all trusted sources are safe to be used for training,
2) at least all uncorrupted sources will be trusted.

Property 1) follows from the fact that if a source has small 
distance to at least half of the other datasets, it must be 
close to at least one of the uncorrupted sources. 
By the property of the discrepancy distance, including it 
in the training set will therefore not
affect the learning of the hypothesis very negatively. 
Property 2) follows from a concentration of mass argument, 
which guarantees that for any uncorrupted source its distance to 
all other uncorrupted sources will approach zero at a well-understood 
rate. Therefore, with a suitably selected threshold, at least 
all uncorrupted sources will be close to each other and
end up in the trusted subset with high probability.

\noindent\textbf{Discrepancy Distance.}
For any dataset $S_i\in(\mathcal{X}\ttimes\mathcal{Y})^m$, let 
\begin{equation}
\widehat{\mathcal{R}}_i(h) = \frac{1}{m}\!\!\sum_{(x,y)\in S_i}\!\!\loss(h(x),y)
\end{equation}
be the empirical risk of a hypothesis $h$ with respect to the loss $\loss$. 
The (empirical) \emph{discrepancy distance} between two datasets, $S_i$
and $S_j$, is defined as 
\begin{equation}\label{eqn:defn_of_emp_disrepancy}
d_{\mathcal{H}}(S_i,S_j) = \sup_{h\in\mathcal{H}}\big(|\widehat{\mathcal{R}}_i(h) - \widehat{\mathcal{R}}_j(h)|\big).
\end{equation}
This is the empirical counterpart of the so-called discrepancy 
distance, which, together with its unsupervised form, is widely adopted within 
the field of domain adaptation~\cite{kifer2004detecting, ben2010theory, mohri2012new}. 
Typically, the discrepancy is used to bound the maximum 
possible effect of distribution drift on a learning system. 
The metric was also used in \cite{konstantinov2019robust} to 
measure the effect of training on sources that have been sampled
randomly, but from adversarially chosen distributions. 
As shown in \citet{kifer2004detecting,ben2010theory}, for randomly 
sampled datasets, the empirical discrepancy concentrates with 
known rates to its distributional value, \ie to zero, if two 
sources have the same underlying data distributions. 
The empirical discrepancy is well-defined even for data not 
sampled from a distribution, though, and together with the 
uniform convergence property it allows us to bound the effect 
of training on one dataset rather than another. 

\noindent\textbf{Step 2.}
Let $S_T=\bigcup_{i\in T}S_i$ be the output of 
the filtering algorithm, \ie the union of all trusted datasets. 
Then, for any $h\in\mathcal{H}$, the empirical risk over $S_T$ 
can be written as 
\begin{equation}\label{eqn:local_empirical_loss}
\widehat{\mathcal{R}}_{\Tau}(h) = \frac{1}{|\Tau|}\sum_{i\in \Tau} \widehat{\mathcal{R}}_i(h)
\end{equation}
We need to show that training on $S_{\Tau}$, \eg by minimizing 
$\widehat{\mathcal{R}}_{\Tau}(h)$, with high probability leads 
to a hypothesis with small risk under the true 
data distribution $\mathcal{D}$.

By construction, we know that for any trusted source $S_i$, there exists an uncorrupted source $S_j$, such that the difference between $\widehat{\mathcal{R}}_i(h)$ and $\widehat{\mathcal{R}}_j(h)$ is bounded by a suitably chosen
constant (that depends on the growth function $s$).
By the uniform convergence property of $\mathcal{H}$, we know 
that for any uncorrupted source, the difference between $\widehat{\mathcal{R}}_i(h)$
and the true risk ${\mathcal{R}}(h)$ can also be bounded in terms
of the growth function $s$. 
In combination, we obtain that $\widehat{\mathcal{R}}_{\Tau}(h)$ 
is a suitably good estimator of the true risk, uniformly over 
all $h\in\mathcal{H}$.
Consequently, $S_{\Tau}$ can be used for successful learning. 

For the formal derivations and, in particular, the choice 
of thresholds, please see the supplemental material.

\section{Hardness of Robust Multi-Source Learning}
\label{sec:lower_bounds}
We now take an orthogonal view compared to 
Section~\ref{sec:upper_bounds}, and study where the hardness of 
the multi-source PAC learning stems from and what allows us to nevertheless overcome it. 
For this, we prove two additional results that describe 
fundamental limits of how well a learner can perform in 
the multi-source adversarial setting. 

For simplicity of exposition we focus on binary classification.
Let $Y = \{-1, 1\}$ and $\loss$ be the zero-one loss, \ie $\loss(y,\bar y)=\llbracket y\neq \bar y\rrbracket$. 
Following \citet{bshouty2002pac}, we define:
\begin{definition}
\label{defn:non_trivial_h}
A hypothesis space $\mathcal{H}$ over an input set $\mathcal{X}$ is 
said to be \textit{non-trivial}, if there exist two points 
$x_1, x_2\in\mathcal{X}$ and two hypotheses $h_1, h_2\in\mathcal{H}$, 
such that $h_1(x_1) = h_2(x_1)$, but $h_1(x_2) \neq h_2(x_2)$.
\end{definition}

\subsection{What makes robust learning possible?} 
We show that if the learner does not make use of the 
multi-source structure of the data, \ie it behaves as a single-source learner
on the union of all data samples, then a (multi-source) fixed-set adversary can always 
\emph{prevent} PAC-learnability.

\begin{theorem}
\label{thm:lower_bound_single_source_learner}
Let $\mathcal{H}$ be a non-trivial hypothesis space. 
Let $m$ and $N$ be any positive integers and let $\good$ be 
a fixed subset of $[N]$ of size $k \in \{1, \ldots, N-1\}$. 
Let $\mathcal{L}:(\mathcal{X}\times\mathcal{Y})^{N\times m} \rightarrow \mathcal{H}$ 
be a multi-source learner that acts by merging the data from all
sources and then calling a single-source learner. Let $S' \in \left(\mathcal{X}\times\mathcal{Y}\right)^{N\times m}$ be drawn \iid from $\mathcal{D}$.
Then there exists a distribution $\mathcal{D}$ 
with $\min_{h\in\mathcal{H}}\mathcal{R}(h)=0$ 
and a fixed-set adversary $\mathcal{A}$ with index set $G$, such that:
\begin{align}
\label{eqn:lower_bound_single_source_learner}
\mathbb{P}_{S'\sim\mathcal{D}} \Big(\mathcal{R}\big(\mathcal{L}(\mathcal{A}(S')\big) > \frac{\alpha}{8(1-\alpha)}  \Big) > \frac{1}{20},
\end{align}
where $\alpha=\frac{N-k}{N}$ is the power of the adversary. 
\end{theorem}
The proof is provided in the supplemental material. Note 
that, since the theorem holds for the fixed-set adversary, 
it automatically also holds for the stronger flexible-set 
adversary.

The theorem sheds light on why PAC-learnability 
is possible in the multi-source setting, while in the single 
source setting it is not. The reason is not simply that the 
adversary is weaker, because it is restricted to manipulating 
samples in a subset of datasets instead of being able to choose
freely. Inequality~\eqref{eqn:lower_bound_single_source_learner} 
implies that even against such a weaker adversary, a single-source 
learner cannot be adversarially robust. 
Consequently, it is the additional information that the data comes in multiple  
datasets, some of which remain uncorrupted even after the 
adversary was active, that gives the multi-source learner 
the power to learn robustly. 

An immediate consequence of Theorem~\ref{thm:lower_bound_single_source_learner}
is also that the common practice of merging the data from all sources 
and performing a form of empirical risk minimization on the resulting
dataset is not a robust learner and therefore suboptimal in the studied context.

\subsection{How hard is robust learning?}
\label{sec:hardness_of_multisource_learning}
As a tool for understanding the limiting factors 
of learning in the adversarial multi-source setting, 
we now establish a lower bound on the achievable 
excess risk in terms of the number of samples 
per source and the power of the adversary.
\begin{theorem}
\label{thm:no_free_lunch}
Let $\mathcal{H} \subset \{h: \mathcal{X}\rightarrow \mathcal{Y}\}$ 
be a hypothesis space, let $m$ and $N$ be any integers and let 
$\good$ be a fixed subset of $[N]$ of size $k \in \{1, \ldots, N-1\}$. Let $S' \in \left(\mathcal{X}\times\mathcal{Y}\right)^{N\times m}$ be drawn \iid from $\mathcal{D}$.
Then the following statements hold for any multi-source learner $\mathcal{L}$:
\begin{itemize}
\item[(a)] Suppose that $\mathcal{H}$ is non-trivial. 
Then there exists a distribution $\mathcal{D}$ on $\mathcal{X}$ 
with $\min_{h\in\mathcal{H}}\mathcal{R}(h)=0$, and 
a fixed-set adversary $\mathcal{A}$ with index set $G$, 
such that:
\begin{align}
\label{eqn:general_lower_bound_1}
\mathbb{P}_{S'} \Big(\mathcal{R}\big(\mathcal{L}(\mathcal{A}(S')\big) > \frac{\alpha}{8m}  \Big) > \frac{1}{20}.
\end{align}
\item[(b)] Suppose that $\mathcal{H}$ has VC dimension $d \geq 2$. 
Then there exists a distribution $\mathcal{D}$ on $\mathcal{X}\times \mathcal{Y}$ and 
a fixed-set adversary $\mathcal{A}$ with index set $G$, such that:
\begin{align}
\label{eqn:general_lower_bound_all}
\mathbb{P}_{S'} \Bigg(&\mathcal{R}\big(\mathcal{L}(\mathcal{A}(S')\big) - \min_{h\in\mathcal{H}}\mathcal{R}(h) 
\\&> \sqrt{\frac{d}{1280Nm}} + \frac{\alpha}{16m} \Bigg) > \frac{1}{64}. \nonumber
\end{align}
\end{itemize}
In both cases, $\alpha=\frac{N-k}{N}$ is the power of the adversary. 
\end{theorem}
The proof is provided in the supplemental material. 
As for Theorem~\ref{thm:lower_bound_single_source_learner}, it 
is clear that the same result holds also for flexible-set adversaries
with preserved size $k$.

\noindent\textbf{Analysis.}
Inequality~\eqref{eqn:general_lower_bound_1} shows that even in 
the realizable scenario, the risk might not shrink faster than 
with rate $\Omega(\alpha/m)$, regardless of how many data 
sources, and therefore data samples, are available. 
This is contrast to the \iid situation, where the corresponding 
rate is $\Omega(1/Nm)$. The difference shows that robust 
learning with a constant fraction of corrupted sources is only 
possible if the number of samples per dataset grows. 
Conversely, if the number of corrupted datasets is constant, 
regardless of the total number of sources, \ie, $\alpha=\mathcal{O}(1/N)$, 
we recover the rates for learning without an adversary up to constants.

In inequality~\eqref{eqn:general_lower_bound_all}, the term 
$\Omega(\sqrt{d/Nm})$ is due to the classic no-free-lunch 
theorem for binary classification and corresponds to the 
fundamental limits of learning, now in the non-realizable case. 
The $\Omega(\alpha/m)$-term appears as the price of robustness,
and as before, it implies that for constant $\alpha$, 
$m\to\infty$ is necessary in order to achieve arbitrarily small excess risk, 
while just $N\to\infty$ does not suffice. 

\noindent\textbf{Relation to prior work.} 
Lower bounds of similar structure as in 
Theorem~\ref{thm:no_free_lunch} have also been derived 
for Byzantine optimization and collaborative learning. 
In particular, \citet{yin2018byzantine} prove that in the case of
distributed mean estimation of a $d$-dimensional Gaussian on $N$ machines, an $\alpha$ 
fraction of which can be Byzantine, any algorithm would 
incur loss of $\Omega(\frac{\alpha}{\sqrt{m}} + \sqrt{\frac{d}{Nm}})$. 
\citet{NIPS2018_7712} construct specific examples of a Lipschitz 
continuous and a strongly convex function, such that no distributed 
stochastic optimization algorithm, working with an $\alpha$-fraction of
Byzantine machines, can optimize the function to error less than 
$\Omega(\frac{\alpha}{\sqrt{m}} + \sqrt{\frac{d}{Nm}})$, where $d$ is the number of parameters. 
For realizable binary classification in the context of collaborative 
learning, \citet{qiao2018outliers} prove that there exists a hypothesis 
space of VC dimension $d$, such that no learner can achieve excess 
risk less than $\Omega(\alpha d/m)$.

Besides the different application scenario, the main difference between 
these results and Theorem \ref{thm:no_free_lunch} is that our 
bounds hold for \textit{any} hypothesis space $\mathcal{H}$ that is non-trivial
(Ineq.~\eqref{eqn:general_lower_bound_1}), or has VC-dimension $d \geq 2$
(Ineq.~\eqref{eqn:general_lower_bound_all}), while the mentioned 
references construct explicit examples of hypothesis spaces or 
stochastic optimization problems where the bounds hold. 
In particular, our results show that the limitations on the learner due 
the finite total number of  samples, the finite number of samples per 
source and the fraction of unreliable sources $\alpha$ are inherent 
and not specific to a subset of hard-to-learn hypotheses.

\section{Conclusion}
\label{sec:conclusion}
We studied the problem of robust learning 
from multiple unreliable datasets.
Rephrasing this task as learning from datasets that 
might be adversarially corrupted, we introduced 
the formal problem of adversarial learning 
from multiple sources, which we studied in the 
classic PAC setting. 

Our main results provide a characterization of the
hardness of this learning task from above and below. 
First, we showed that adversarial multi-source 
PAC learning is possible for any hypothesis class with the uniform 
convergence property, and we provided explicit rates 
for the excess risk (Theorem~\ref{thm:upper_bound} 
and Corollaries). The proof is constructive and shows also 
that integrating robustness comes at a minor statistical 
cost, as our robust learner achieves optimal 
rates when run on data without manipulations.
Second, we proved that adversarial PAC learning 
from multiple sources is far from trivial. In particular, 
it is impossible to achieve for learners that ignore the 
multi-source structure of the data (Theorem~\ref{thm:lower_bound_single_source_learner}).
Third, we proved lower bounds on the excess risk under 
very general conditions (Theorem~\ref{thm:no_free_lunch}), 
which highlight an unavoidable slowdown of the convergence rate 
proportional to the adversary's strength compared 
to the \iid (adversarial-free) case. 
Furthermore, in order to facilitate successful learning
with a constant fraction of corrupted sources, 
the number of samples per source has to grow. 

A second emphasis of our work was to highlight connections of the 
adversarial multi-source learning task to related methods in 
robust optimization, cryptography and statistics. 
We believe that a better understanding of these connections will 
allow us to come up with tighter bounds and to design algorithms 
that are not only statistically efficient (as was the focus of 
this work), but also obtain insight into the trade-offs with 
computational complexity.

\section*{Acknowledgements}

Dan Alistarh is supported in part by the European Research Council (ERC) under the European Union's Horizon 2020 research and innovation programme (grant agreement No 805223 ScaleML). This research was supported by the Scientific Service Units (SSU) of IST Austria through resources provided by Scientific Computing (SciComp).

\bibliography{ms}
\bibliographystyle{icml2020}

\onecolumn

\appendix

\clearpage
\section{Proof of Theorem \ref{thm:upper_bound} and its corollaries}
\label{app:upper_bound_proof}
\addtocounter{algorithm}{-1}
\addtocounter{theorem}{-3}
\addtocounter{corollary}{-2}

\begin{theorem}
\label{thm:upper_bound_appendix}
Let $N, m, k \in \mathbb{N}$ be integers, such that $k \in (N/2, N]$. 
Let $\alpha = \frac{N-k}{N} < \frac{1}{2}$ be the proportion of corrupted sources. 
Assume that $\mathcal{H}$ has the uniform convergence property with 
rate function $s$. 
Then there exists a learner 
$\mathcal{L}: \left(\mathcal{X}\ttimes\mathcal{Y}\right)^{N \times m} \rightarrow \mathcal{H}$ 
with the following two properties.
\begin{itemize}
\item[(a)] Let $\good$ be a fixed subset of $[N]$ of size $|\good| = k$. 
For $S^{'} = \{S^{'}_1, \ldots, S^{'}_N\}\stackrel{\iid}{\sim}\mathcal{D}$, with probability at least $1-\delta$ over the sampling of $S'$:
\begin{align}
\mathcal{R}(\mathcal{L}(\mathcal{A}(S^{'}))) - \min_{h\in\mathcal{H}}\mathcal{R}(h) \leq 2s\big(km, \frac{\delta}{2}, S_{\good}\big) + 6\alpha\max_{i \in [N]}s\big(m, \frac{\delta}{2N}, S_{i}\big),
\end{align}
uniformly against all fixed-set adversaries with preserved set $\good$, where $S = \{S_1, \ldots, S_N\} = \mathcal{A}(S^{'})$ is
the dataset modified the adversary and $S_{\good} = \cup_{i\in\good}S_i$ is the set of all uncorrupted data.
\item[(b)] 
For $S^{'} = \{S^{'}_1, \ldots, S^{'}_N\}\stackrel{\iid}{\sim}\mathcal{D}$, with probability at least $1-\delta$ over the sampling of $S'$:
\begin{align}
\label{eqn:no_trusted_data_flexible-set_appendix}
\mathcal{R}(\mathcal{L}(\mathcal{A}(S^{'}))) - \min_{h\in\mathcal{H}}\mathcal{R}(h) \leq 2s\big(km, \frac{\delta}{2 \binom{N}{k}}, S_{\good}\big) + 6 \alpha \max_{i\in [N]} s\big(m, \frac{\delta}{2N}, S_i\big),
\end{align}
uniformly against all flexible-set adversaries with preserved size $k$, where $S = \{S_1, \ldots, S_N\} = \mathcal{A}(S^{'})$ is the dataset returned by the adversary, $\good$ is the set of sources not modified by the adversary and $S_{\good} = \cup_{i\in\good}S_i$ is the set of all uncorrupted data.

\end{itemize}
\end{theorem}
    
\begin{proof}
Denote by $S'_{i} = \{(x'_{i,1}, y'_{i,1}), \ldots, (x'_{i,m}, y'_{i,m})\}$ for $i = 1, \ldots, N$ the initial datasets and by $S_{i} = \{(x_{i,1}, y_{i,1}), \ldots, (x_{i,m}, y_{i,m})\}$ for $i = 1, \ldots, N$ the datasets after the modifications of the adversary. As explained in the main body of the paper, we denote by:
\begin{equation}
\widehat{\mathcal{R}}_i(h) = \frac{1}{m}\sum_{j=1}^m \loss(h(x_{i,j}), y_{i,j})
\end{equation}
the empirical risk of any hypothesis $h\in\mathcal{H}$ on the dataset $S_i$ and by:
\begin{equation}
d_{\mathcal{H}}(S_i, S_j) = \sup_{h\in\mathcal{H}} |\widehat{\mathcal{R}}_i(h) - \widehat{\mathcal{R}}_j(h)|
\end{equation}
the empirical discrepancy between the datasets $S_i$ and $S_j$. 

We show that a learner that first runs a certain filtering algorithm (Algorithm \ref{alg:filtering}) based on the discrepancy metric and then performs empirical risk minimization on the remaining data to compute a hypothesis satisfies the properties stated in the theorem. The full algorithm for the learner is therefore given in Algorithm \ref{alg:algo_full}.

\begin{algorithm}[b]
\caption{} \label{alg:filtering}
\begin{algorithmic}
\INPUT Datasets $S_1, \ldots, S_N$
\STATE  Initialize $\Tau = \{\}$ \qquad // trusted sources
\FOR{$i=1,\dots,N$}
\IF{$d_{\mathcal{H}}\big(S_i, S_j\big) \leq s \left(m, \frac{\delta}{2N}, S_i\right) + s \left(m, \frac{\delta}{2N}, S_j\right),$\\
\qquad for at least $\lfloor \frac{N}{2} \rfloor$ values of $j\not = i$,}
\STATE{$\Tau = \Tau \cup \{i\}$}  
\ENDIF
\ENDFOR
\OUTPUT $\bigcup_{i\in \Tau} S_i$ \qquad // indices of the trusted sources 
\end{algorithmic}
\end{algorithm}

\begin{algorithm}[t]
\caption{} \label{alg:algo_full}
\begin{algorithmic}
\INPUT Datasets $S_1, \ldots, S_N$
\STATE  Run Algorithm \ref{alg:filtering} to compute $S_{\Tau} = \bigcup_{i\in \Tau} S_i$
\STATE Compute $h^{\mathcal{A}} = \argmin_{h\in\mathcal{H}} \frac{1}{|S_{\Tau}|}\sum_{(x,y)\in S_{\Tau}} \loss(h(x), y)$ \qquad //empirical risk minimizer over all trusted sources 
\OUTPUT $h^{\mathcal{A}}$ 
\end{algorithmic}
\end{algorithm}
(a) The key idea of the proof is that the clean sources are close to each other with high probability, so they get selected when running Algorithm \ref{alg:filtering}. On the other hand, if a bad source has been selected, it must be close to \textit{at least one} of the good sources, so it can not have too bad an effect on the empirical risk.

For all $i \in \good$, let $\mathcal{E}_i$ be the event that:
\begin{equation}
\label{eqn:concentration_sources_fixed-set2}
\sup_{h\in\mathcal{H}} \left|\mathcal{R}(h) - \widehat{\mathcal{R}}_i(h)\right| \leq s\left(m, \frac{\delta}{2N}, S_i\right).
\end{equation}
Further, let $\mathcal{E}_{\good}$ be the event that:
\begin{equation}
\label{eqn:concentration_good_fixed-set2}
\sup_{h\in\mathcal{H}} \left|\mathcal{R}(h) - \widehat{\mathcal{R}}_{\good}(h)\right| \leq s\left(km, \frac{\delta}{2}, S_{\good}\right),
\end{equation}
where $$\widehat{\mathcal{R}}_{G}(h) = \frac{1}{km}\sum_{i \in \good} \sum_{j=1}^m \loss (h(x_{i,j}), y_{i,j}).$$
Denote by $\mathcal{E}_i^c$ and $\mathcal{E}_{\good}^c$ the complements of these events. Then we know that $\mathbb{P}\left(\mathcal{E}_{\good}^{c}\right) \leq \frac{\delta}{2}$, and $\mathbb{P}\left(\mathcal{E}_i^{c}\right) \leq \frac{\delta}{2N}$ for all $i \in \good$. Therefore, if $\mathcal{E} = \mathcal{E}_{\good} \land \left(\land_{i\in\good} \mathcal{E}_i\right)$, we have:
\begin{equation}
\mathbb{P}\left(\mathcal{E}^c\right) = \mathbb{P}\left(\mathcal{E}_{\good}^c \lor \left(\lor_{i\in\good} \mathcal{E}_i^c\right)\right) \leq \mathbb{P}\left(\mathcal{E}_{\good}^c\right) + \sum_{i\in\good}\mathbb{P}\left(\mathcal{E}_i^c\right)\leq \frac{\delta}{2} + k\frac{\delta}{2N} \leq \delta.
\end{equation}
Hence, the probability of the event $\mathcal{E}$ that all of (\ref{eqn:concentration_sources_fixed-set2}) and (\ref{eqn:concentration_good_fixed-set2}) hold, is at least $1 - \delta$. We now show that under the event $\mathcal{E}$, Algorithm \ref{alg:algo_full} returns a hypothesis that satisfies the condition in (a).

Indeed, fix an arbitrary fixed-set adversary $\mathcal{A}$ with preserved set $\good$. Whenever $\mathcal{E}$ holds, for all $i, j \in \good$ we have:
\begin{equation}
\begin{split}
d_{\mathcal{H}}\left(S_i, S_j\right) = \sup_{h\in\mathcal{H}}(|\widehat{\mathcal{R}}_i(h) - \widehat{\mathcal{R}}_j(h)|) & \leq \sup_{h\in\mathcal{H}} \left(|\widehat{\mathcal{R}}_i(h) - \mathcal{R}(h)|\right) + \sup_{h\in\mathcal{H}}\left(|\mathcal{R}(h) - \widehat{\mathcal{R}}_j(h)|\right) \\ & \leq s \left(m, \frac{\delta}{2N}, S_i\right) + s \left(m, \frac{\delta}{2N}, S_j\right).
\end{split}
\end{equation}
Now since $k \geq \lfloor \frac{N}{2} \rfloor + 1$, we get that $\good \subseteq \Tau$. Moreover, for any $i \in \Tau\setminus\good$, there exists \textit{at least one} $j \in \good$, such that $d_{\mathcal{H}}(S_i, S_j) \leq s \left(m, \frac{\delta}{2N}, S_i\right) + s \left(m, \frac{\delta}{2N}, S_j\right)$. For any $i\in\Tau \setminus \good$, denote by $f(i)$ the smallest such $j$. Therefore, for any $i \in (\Tau \setminus \good)$:
\begin{align}
|\widehat{\mathcal{R}}_i(h) - \mathcal{R}(h)| \leq |\widehat{\mathcal{R}}_i - \widehat{\mathcal{R}}_{f(i)}(h)| + |\widehat{\mathcal{R}}_{f(i)}(h) - \mathcal{R}(h)| & \leq d_{\mathcal{H}}\left(S_i, S_{f(i)}\right) + s\left(m, \frac{\delta}{2N}, S_{f(i)}\right) \\ & \leq s\left(m, \frac{\delta}{2N}, S_{i}\right) + 2s\left(m, \frac{\delta}{2N}, S_{f(i)}\right)
\end{align}
Denote by
\begin{equation}
\widehat{\mathcal{R}}_{\Tau}(h) = \frac{1}{|\Tau|} \sum_{i\in\Tau} \widehat{\mathcal{R}}_{i}(h) = \frac{1}{|S_{\Tau}|} \sum_{(x,y)\in S_{\Tau}} \loss (h(x),y)
\end{equation}
the loss over all the trusted data. Then for any $h\in\mathcal{H}$ we have:
\begin{align}
\left|\widehat{\mathcal{R}}_{\Tau}(h) - \mathcal{R}(h)\right| & \leq \frac{1}{|\Tau|m} \left(\left|\sum_{i\in\good}\sum_{l=1}^m \left(\loss(h(x_{i,l}), y_{i,l}) - \mathcal{R}(h)\right)\right| + \sum_{i\in(\Tau \setminus \good)}\left|\sum_{l=1}^m \left(\loss(h(x_{i,l}), y_{i,l}) -  \mathcal{R}(h)\right)\right|\right) \\ 
& = \frac{k}{|\Tau|}\left|\widehat{\mathcal{R}}_{\good}(h) -  \mathcal{R}(h)\right| + \frac{1}{|\Tau|}\sum_{i\in(\Tau \setminus \good)} \left|\widehat{\mathcal{R}}_i(h) - \mathcal{R}(h)\right| \\
 & \leq \frac{k}{|\Tau|} s\left(km, \frac{\delta}{2}, S_{\good}\right) + \frac{1}{|\Tau|}\sum_{i\in(\Tau \setminus \good)} \left|\widehat{\mathcal{R}}_i(h) - \mathcal{R}(h)\right| \\
& \leq \frac{k}{|\Tau|} s\left(km, \frac{\delta}{2}, S_{\good}\right) + \frac{1}{|\Tau|} \sum_{i\in(\Tau \setminus \good)} \left(s\left(m, \frac{\delta}{2N}, S_{i}\right) + 2s\left(m, \frac{\delta}{2N}, S_{f(i)}\right)\right) \\ 
& \leq \frac{k}{|\Tau|} s\left(km, \frac{\delta}{2}, S_{\good}\right) + 3\frac{|\Tau| - k}{|\Tau|}\max_{i \in [N]}s\left(m, \frac{\delta}{2N}, S_{i}\right)\\ 
& \leq  s\left(km, \frac{\delta}{2}, S_{\good}\right) + 3\frac{N-k}{N}\max_{i \in [N]}s\left(m, \frac{\delta}{2N}, S_{i}\right)
\end{align}
Finally, let $h^{*} = \argmin_{h\in\mathcal{H}} \mathcal{R}(h)$ and $h^{\mathcal{A}} = \mathcal{L}(\mathcal{A}(S')) = \argmin_{h\in\mathcal{H}}\widehat{\mathcal{R}}_{\Tau}(h)$. Then:
\begin{align}
\mathcal{R}(h^{\mathcal{A}}) - \mathcal{R}(h^{*}) = \left(\mathcal{R}(h^{\mathcal{A}}) - \widehat{\mathcal{R}}_{\Tau}(h^{\mathcal{A}})\right) + \left(\widehat{\mathcal{R}}_{\Tau}(h^{\mathcal{A}}) - \mathcal{R}(h^{*})\right) & \leq \left(\mathcal{R}(h^{\mathcal{A}}) - \widehat{\mathcal{R}}_{\Tau}(h^{\mathcal{A}})\right) + \left(\widehat{\mathcal{R}}_{\Tau}(h^{*}) - \mathcal{R}(h^{*})\right) \\
   & \leq 2 \sup_{h\in\mathcal{H}}\left|\widehat{\mathcal{R}}_{\Tau}(h) - \mathcal{R}(h)\right|.
\end{align}
Since we showed this result for an arbitrary fixed-set adversary with preserved set $\good$, the result follows.

(b) The crucial difference in the case of the flexible-set adversary is that the set $\good$ is chosen after the clean data is observed. We thus need concentration results for \textit{all} of the subsets of $[N]$ of size $k$, as well as all individual sources.

For all $i \in [N]$, let $\mathcal{E}_i$ be the event that:
\begin{equation}
\label{eqn:concentration_sources_flexible-set2}
\sup_{h\in\mathcal{H}} \left|\mathcal{R}(h) - \widehat{\mathcal{R}}'_i(h)\right| \leq s\left(m, \frac{\delta}{2N}, S'_i\right),
\end{equation}
where 
\begin{equation}
\widehat{\mathcal{R}}'_{i} = \frac{1}{m}\sum_{j=1}^m \loss (h(x'_{i,j}), y'_{i,j})
\end{equation}
Further, for \textit{any} $A \subseteq [N]$ of size $|A| = k$, let $\mathcal{E}_{A}$ be the event that:
\begin{equation}
\label{eqn:concentration_good_flexible-set2}
\sup_{h\in\mathcal{H}} \left|\mathcal{R}(h) - \widehat{\mathcal{R}}'_{A}(h)\right| \leq s\left(km, \frac{\delta}{2 {N\choose k}}, S'_{A}\right),
\end{equation}
where $S'_A = \cup_{i\in A} S'_i$ and 
\begin{equation}
\widehat{\mathcal{R}}'_{A}(h) = \frac{1}{km}\sum_{i\in A}\sum_{l=1}^m \loss(h(x'_{i,l}), y'_{i,l}).
\end{equation}
Then we know that $\mathbb{P}\left(\mathcal{E}_i^{c}\right) \leq \frac{\delta}{2N}$ for all $i \in [N]$ and $\mathbb{P}\left(\mathcal{E}_{G}^{c}\right) \leq \frac{\delta}{2\binom{N}{k}}$ for all $A \subseteq [N]$ with $|A| = k$. Therefore, if $\mathcal{E} = \left(\land_{A} \mathcal{E}_{A}\right) \land \left(\land_{i\in[N]} \mathcal{E}_i\right)$, we have:
\begin{align}
\mathbb{P}\left(\mathcal{E}^c\right) = \mathbb{P}\left(\left(\lor_{A}\mathcal{E}_{A}^c\right) \lor \left(\lor_{i\in [N]} \mathcal{E}_i^c\right)\right) \leq \sum_{A}\mathbb{P}\left(\mathcal{E}_{A}^c\right) + \sum_{i\in [N]}\mathbb{P}\left(\mathcal{E}_i^c\right) \leq \binom{N}{k}\frac{\delta}{2{N\choose k}} + N\frac{\delta}{2N} = \delta.
\end{align}
Hence, the probability of the event $\mathcal{E}$ that all of (\ref{eqn:concentration_sources_flexible-set2}) and (\ref{eqn:concentration_good_flexible-set2}) hold, is at least $1 - \delta$. In particular, under $\mathcal{E}$:
\begin{equation}
\sup_{h\in\mathcal{H}} \left|\mathcal{R}(h) -\widehat{\mathcal{R}}_{\good}(h)\right| = \sup_{h\in\mathcal{H}} \left|\mathcal{R}(h) -\widehat{\mathcal{R}}'_{\good}(h)\right| \leq s\left(km, \frac{\delta}{2 \binom{N}{k}}, S'_{\good}\right) = s\left(km, \frac{\delta}{2 \binom{N}{k}}, S_{\good}\right)
\end{equation}
and 
\begin{equation}
\sup_{h\in\mathcal{H}} \left|\mathcal{R}(h) - \widehat{\mathcal{R}}_i(h)\right| = \sup_{h\in\mathcal{H}} \left|\mathcal{R}(h) - \widehat{\mathcal{R}}'_i(h)\right| \leq s\left(m, \frac{\delta}{2N}, S'_i\right) = s\left(m, \frac{\delta}{2N}, S_i\right),
\end{equation}
for all $i \in \good$.

Now, for any flexible-set adversary with preserved size $k$, the same argument as in (a) shows that:
\begin{equation}
\label{eqn:trusted_data_flexible-set}
\begin{split}
& \mathcal{R}(h^{\mathcal{A}}) - \mathcal{R}(h^{*}) \leq 2s\left(km, \frac{\delta}{2 \binom{N}{k}}, S_{\good}\right) + 6 \frac{N-k}{N} \max_{i\in [N]} s\left(m, \frac{\delta}{2N}, S_i\right)
\end{split}
\end{equation}
holds under the event $\mathcal{E}$.
\end{proof}

We now show how to obtain data-dependent guarantees, via the notion of Rademacher complexity. Let
\begin{equation}
\mathfrak{R}_S\left(\loss \circ \mathcal{H}\right) = \mathbb{E}_{\sigma}\left(\sup_{h\in\mathcal{H}} \frac{1}{n} \sum_{i=1}^{n}\sigma_i \loss(h(x_i), y_i)\right)
\end{equation}
be the Rademacher complexity of $\mathcal{H}$ with respect to the loss function $\loss$ on a sample $S = \{(x_1, y_1), \ldots, (x_n, y_n)\}$. Let $S_{\good} = \cup_{i\in\good}S_i$, $\mathfrak{R}_i = \mathfrak{R}_{S_i}\left(\loss \circ \mathcal{H}\right)$ and $\mathfrak{R}_{\good} = \mathfrak{R}_{S_{\good}}\left(\loss \circ \mathcal{H}\right)$. Then we have:

\begin{corollary}
\label{cor:rademacher_rates1}
In the setup of Theorem \ref{thm:upper_bound},
against a fixed-set adversary, it holds that 
\begin{align}
\label{eqn:fixed-set_rademacher-appendix}
\mathcal{R}(\mathcal{L}(\mathcal{A}(S^{'}))) - \min_{h\in\mathcal{H}}\mathcal{R}(h)  
\leq 4\mathfrak{R}_{\good} + 6\sqrt{\frac{\log(\frac{4}{\delta})}{2km}} + \alpha\Big(18\sqrt{\frac{\log\left(\frac{4N}{\delta}\right)}{2m}} + 12\max_{i\in [N]}\mathfrak{R}_i\Big). 
\end{align}
\end{corollary}

\begin{proof}
We use the standard generalization bound based on Rademacher complexity. Assume that $S = \{\left(x_1, y_1\right), \ldots, \left(x_n, y_n\right)\} \sim \mathcal{D}$, then with probability at least $1 - \delta$ over the data \cite{mohri2018foundations}:
\begin{equation}
\label{eqn:rademacher_concentration}
\sup_{h\in\mathcal{H}} |\mathbb{E}\left(\loss(h(x), y)\right) - \frac{1}{n}\sum_{i=1}^n \loss(h(x_i), y_i)| \leq 2 \mathfrak{R}_S \left(\loss \circ \mathcal{H}\right) + 3\sqrt{\frac{\log\left(\frac{2}{\delta}\right)}{2n}}.
\end{equation}
Substituting into the result of Theorem \ref{thm:upper_bound} gives the result.
\end{proof}

\begin{corollary}
\label{cor:rademacher_rates2}
In the setup of Theorem \ref{thm:upper_bound}, against a flexible-set 
adversary, it holds that 
\begin{align}
\label{eqn:malicious_rademacher-appendix}
\mathcal{R}(\mathcal{L}(\mathcal{A}(S^{'}))) - \min_{h\in\mathcal{H}}\mathcal{R}(h)  
 \leq 4\mathfrak{R}_{\good} + 12\alpha \max_{i\in [N]}\mathfrak{R}_i + \widetilde{\mathcal{O}}\left(\frac{\sqrt[4]{\alpha}}{\sqrt{m}}\right). 
\end{align}
\end{corollary}

\begin{proof}
Using the concentration result from Corollary \ref{cor:rademacher_rates1} and $\binom{N}{k} = \binom{N}{(1 - \alpha)N} = \binom{N}{\alpha N} \leq 2^{H(\alpha) N}$, where $H(p) = -p \log_2(p) - (1-p)\log_2(1-p)$ is the binary entropy function, we obtain:
\begin{align}
\mathcal{R}(\mathcal{L}(\mathcal{A}(S'))) - \min_{h\in\mathcal{H}} \mathcal{R}(h) & \leq 4\mathfrak{R}_{\good} + 6\sqrt{\frac{\log(\frac{4\binom{N}{k}}{\delta})}{2km}} + \alpha\left(18\sqrt{\frac{\log\left(\frac{4N}{\delta}\right)}{2m}} + 12\max_{i\in [N]}\mathfrak{R}_i\right) 
\\ & = 4\mathfrak{R}_{\good} + 6\sqrt{\frac{\log(\binom{N}{k})}{2km} + \frac{\log(\frac{4}{\delta})}{2km}} + \alpha\left(18\sqrt{\frac{\log\left(\frac{4N}{\delta}\right)}{2m}} + 12\max_{i\in [N]}\mathfrak{R}_i\right) \\ & \leq 4\mathfrak{R}_{\good} + 6\sqrt{\frac{H(\alpha)N \log(2)}{2(1-\alpha)Nm} + \frac{\log(\frac{4}{\delta})}{2(1-\alpha)Nm}} + \alpha\left(18\sqrt{\frac{\log\left(\frac{4N}{\delta}\right)}{2m}} + 12\max_{i\in [N]}\mathfrak{R}_i\right)  
\\ & \leq 4\mathfrak{R}_{\good} + 12\alpha \max_{i\in [N]}\mathfrak{R}_i + \widetilde{\mathcal{O}}\left(\frac{\sqrt[4]{\alpha}}{\sqrt{m}}\right)
\end{align}
where for the last inequality we used $H(\alpha) \leq 2\sqrt{\alpha(1-\alpha)}$, $1-\alpha \in (\frac{1}{2}, 1]$ and $\sqrt[4]{\alpha} > \alpha$.
\end{proof}

For the case of binary classifiers, we also provide a simpler bound in terms of the VC dimension of $\mathcal{H}$.
\begin{corollary}
\label{cor:vc_rates}
Assume that $Y = \{-1, 1\}$ and that $\mathcal{H}$ has finite VC-dimension $d$. Then:
\begin{itemize}
\item[(a)] In the case of the fixed-set adversary there exists a universal constant $C$, such that:
\begin{equation}
\label{eqn:fixed-set_vc}
\begin{split}
\mathcal{R}(\mathcal{L}(\mathcal{A}(S'))) - \min_{h\in\mathcal{H}} \mathcal{R}(h) \leq 2C\sqrt{\frac{d}{km}} + 2\sqrt{\frac{2\log(\frac{4}{\delta})}{km}} + \alpha\left(6C\sqrt{\frac{d}{m}} + 6\sqrt{\frac{2\log\left(\frac{4N}{\delta}\right)}{m}}\right).
\end{split}
\end{equation}

\item[(b)] In the case of the flexible-set adversary:
\begin{equation}
\label{eqn:flexible-set_vc}
\begin{split}
\mathcal{R}(\mathcal{L}(\mathcal{A}(S'))) - \min_{h\in\mathcal{H}} \mathcal{R}(h) \leq \mathcal{O}\left(\sqrt{\frac{d}{km}} + \frac{\sqrt[4]{\alpha}}{\sqrt{m}} + \alpha \sqrt{\frac{d}{m}} + \alpha \sqrt{\frac{\log(N)}{m}} \right).
\end{split}
\end{equation}
\end{itemize}
\end{corollary}

\begin{proof}
(a) Whenever $\mathcal{H}$ is of finite VC-dimension $d$, there exists a constant $C$, such that the following generalization bound holds \cite{bousquet2004introduction}:
\begin{equation}
\label{eqn:vc_concentration_nonrealizable}
\sup_{h\in\mathcal{H}} |\mathbb{E}\left(\loss(h(x), y)\right) - \frac{1}{n}\sum_{i=1}^n \loss(h(x_i), y_i)| \leq C\sqrt{\frac{d}{n}} + \sqrt{\frac{2\log\left(\frac{2}{\delta}\right)}{n}}
\end{equation}
and hence $\mathcal{H}$ has the uniform convergence property with rate function $s = C\sqrt{\frac{d}{n}} + \sqrt{\frac{2\log\left(\frac{2}{\delta}\right)}{n}}$. Substituting into the result of Theorem \ref{thm:upper_bound} gives the result.

(b) Using the concentration result from (a) and $\binom{N}{k} = \binom{N}{(1 - \alpha)N} = \binom{N}{\alpha N} \leq 2^{H(\alpha) N}$, where $H(p) = -p \log_2(p) - (1-p)\log_2(1-p)$ is the binary entropy function, we obtain:
\begin{align}
\mathcal{R}(\mathcal{L}(\mathcal{A}(S'))) - \min_{h\in\mathcal{H}} \mathcal{R}(h) & \leq 2C\sqrt{\frac{d}{km}} + 2\sqrt{\frac{2\log(\frac{4 \binom{N}{k}}{\delta})}{km}} + \alpha\left(6C\sqrt{\frac{d}{m}} + 6\sqrt{\frac{2\log\left(\frac{4N}{\delta}\right)}{m}}\right) \\
 & = 2C\sqrt{\frac{d}{km}} + 2\sqrt{\frac{2\log(\binom{N}{k})}{km} + \frac{2\log(\frac{4}{\delta})}{km}} + \alpha\left(6C\sqrt{\frac{d}{m}} + 6\sqrt{\frac{2\log\left(\frac{4N}{\delta}\right)}{m}} \right) \\
  & \leq 2C\sqrt{\frac{d}{km}} + 2\sqrt{\frac{2H(\alpha)N \log(2)}{(1-\alpha)Nm} + \frac{2\log(\frac{4}{\delta})}{(1-\alpha)Nm}} + \alpha\left(6C\sqrt{\frac{d}{m}} + 6\sqrt{\frac{2\log\left(\frac{4N}{\delta}\right)}{m}} \right) \\
   & \leq \mathcal{O}\left(\sqrt{\frac{d}{km}} + \frac{\sqrt[4]{\alpha}}{\sqrt{m}} + \alpha \sqrt{\frac{d}{m}} + \alpha \sqrt{\frac{\log(N)}{m}}\right),
\end{align}
where for the last inequality we used $H(\alpha) \leq 2\sqrt{\alpha(1-\alpha)}$ and $1-\alpha \in (\frac{1}{2}, 1]$.
\end{proof}

\clearpage

\section{Proof of Theorem \ref{thm:lower_bound_single_source_learner}}
\label{app:no_free_lunch_single_source_proof}
\begin{theorem}
\label{thm:lower_bound_single_source_learner}
Let $\mathcal{H}$ be a non-trivial hypothesis space. 
Let $m$ and $N$ be any positive integers and let $\good$ be 
a fixed subset of $[N]$ of size $k \in \{1, \ldots, N-1\}$. 
Let $\mathcal{L}:(\mathcal{X}\times\mathcal{Y})^{N\times m} \rightarrow \mathcal{H}$ 
be a multi-source learner that acts by merging the data from all
sources and then calling a single-source learner. Let $S' \in \left(\mathcal{X}\times\mathcal{Y}\right)^{N\times m}$ be drawn \iid from $\mathcal{D}$.
Then there exists a distribution $\mathcal{D}$ 
with $\min_{h\in\mathcal{H}}\mathcal{R}(h)=0$ 
and a fixed-set adversary $\mathcal{A}$ with index set $G$, such that:
\begin{align}
\label{eqn:lower_bound_single_source_learner-appendix}
\mathbb{P}_{S'\sim\mathcal{D}} \Big(\mathcal{R}\big(\mathcal{L}(\mathcal{A}(S')\big) > \frac{\alpha}{8(1-\alpha)}  \Big) > \frac{1}{20},
\end{align}
where $\alpha=\frac{N-k}{N}$ is the power of the adversary. 
\end{theorem}

We use a similar proof technique as in the no-free-lunch results in \cite{bshouty2002pac} and in the classic no-free-lunch theorem, \eg Theorem 3.20 in \cite{mohri2018foundations}. An overview is as follows. Consider a distribution on $\mathcal{X}$ that has support only at two points - the common point $x_1$ and the rare point $x_2$. Take $\mathbb{P}(x_2) = \mathcal{O}(\frac{\alpha}{1-\alpha})$. Then the expected number of occurrences of the point $x_2$ in $\good$ is $\mathcal{O}\left(\frac{\alpha}{1-\alpha} (1-\alpha)Nm\right) = \mathcal{O}\left(\alpha Nm\right)$. Thus, one can show that with constant probability the number of $x_2$'s in $\good$ is at most $\alpha Nm$ and hence the adversary (that has access to exactly $\alpha Nm$ points in total) can insert the same number of $x_2$'s, but wrongly labelled, into the final dataset. Therefore, based on the union of the corrupted datasets, no algorithm can guess with probability greater than $1/2$ what the true label of $x_2$ was. 

\begin{proof}
We prove that there exists a distribution $\mathcal{D}$ on $\mathcal{X}$ and a labelling function $f\in\mathcal{H}$, such that the resulting joint distribution on $\mathcal{X}\times\mathcal{Y}$, defined by $x\sim \mathcal{D}$ and $y = f(x)$, satisfies the desired property.
  
Without loss of generality, let $\good = [1, 2, \ldots, k]$. Since $\mathcal{H}$ is non-trivial, there exist $h_1, h_2 \in \mathcal{H}$ and $x_1, x_2 \in\mathcal{X}$, such that $h_1(x_1) = h_2(x_1)$, while $h_1(x_2) = 1$, but $h_2(x_2) = -1$. Consider the following distribution on $\mathcal{X}$:
\begin{equation}
\mathbb{P}_{\mathcal{D}}(x_1) = 1 - 4\epsilon \quad \text{and} \quad \mathbb{P}_{\mathcal{D}}(x_2) = 4\epsilon,
\end{equation}
where $\epsilon = \frac{1}{8}\frac{\alpha}{1-\alpha}$. Assume that the points are labelled by a function $f\in\mathcal{H}$ (to be chosen later as either $h_1$ or $h_2$). Denote the initial uncorrupted collection of datasets by $S' = (S'_1, \ldots, S'_N)$, with $S'_i = \{(x'_{i,1}, f(x'_{i,1})), \ldots, (x'_{i,m}, f(x'_{i,m}))\}$ and $x'_{i,j}$ being  \iid samples from $\mathcal{D}$.

First we show that with constant probability the point $x_2$ appears at most $\alpha Nm$ times in $\good$. Indeed, let $C$ be this number of appearances. Then $C$ is a binomial random variable with probability of success $4\epsilon$ and number of trials $(1- \alpha)Nm$. Therefore, by the Chernoff bound:
\begin{equation}
\mathbb{P}_{S'}(C \geq \alpha Nm) = \mathbb{P}_{S'}(C \geq (1+1)4\epsilon (1 - \alpha) Nm) \leq e^{-\frac{\alpha Nm}{6}} \leq e^{-1/6} < \frac{17}{20}
\end{equation}
and so:
\begin{equation}
\mathbb{P}_{S'}(C \leq \alpha Nm) > \frac{3}{20}.
\end{equation}
Now consider the following policy for the fixed-set adversary $\mathcal{A}^{s}: S' \rightarrow S$. For any index $i\in [N]$ the adversary replaces $S'_i = \{(x'_{i, 1}, f(x'_{i, 1})), \ldots, (x'_{i, m}, f(x'_{i,m}))\}$ with a dataset $S_i = \{(x_{i, 1}, y_{i, 1}) \ldots, (x_{i, m}, y_{i, m})\}$, such that:

\begin{align}
    (x_{i,j}, y_{i, j})= 
\begin{cases}
    (x'_{i,j}, f(x'_{i, j})),& \text{if } i\in\good = [1, 2, \ldots, k]\\
    (x_2, - f(x_2)), & \text{if } i\in [k+1, \ldots, N] \text{ and } (i-k-1)m + j \leq C \\
    (x_1, f(x_1)), & \text{otherwise} \\
\end{cases}
\end{align}
Then the adversary returns $S = (S_1, \ldots, S_N)$. That is, the adversary keeps the datasets in $\good$ untouched, and fills the datasets in $[N]\backslash \good$ with as many $x_2$'s as there are in $\good$, but wrongly labelled.

Crucially, whenever $C \leq \alpha Nm$, the union of the data in all $N$ sets will look the same no matter if the original labelling function was $h_1$ or $h_2$. In particular, $\mathcal{L}(\mathcal{A}^s(S'))$ will be identical in both cases.

Finally, we argue that under the event $C \leq \alpha Nm$ and the chosen adversary, the learner would incur high loss and show that this implies the result in (\ref{eqn:lower_bound_single_source_learner}). Let $\mathcal{S}$ be the set of all datasets in $\left(\mathcal{X}\times\mathcal{Y}\right)^{N\times m}$, such that $C \leq \alpha Nm$ holds. We just showed that $\mathbb{P}_{S'}(S' \in\mathcal{S}) > \frac{3}{20}$ and that whenever $S' \in\mathcal{S}$, $\mathcal{L}(\mathcal{A}^s(S'))$ is independent of whether the original labelling function was $h_1$ or $h_2$.

Consider a fixed set $S'\in\mathcal{S}$ and let $S = \mathcal{A}^s(S')$ and $h_S = \mathcal{L}(S)$. Denote by $\mathcal{R}(h_S, f) = \mathbb{P}_{\mathcal{D}}(h_S(x) \neq f(x) \cap x \neq x_1)$ and note that $\mathcal{R}(h_S, f) \leq \mathbb{P}_{\mathcal{D}}(h_S(x) \neq f(x)) = \mathcal{R}(\mathcal{L}(\mathcal{A}^s(S')))$. Notice that:

\begin{align}
\mathcal{R}(h_S, h_1) + \mathcal{R}(h_S, h_2) & = \sum_{i=1,2} \mathbbm{1}_{h_S(x_i) \neq h_1(x_i)}\mathbbm{1}_{x_i \neq x_1} \mathbb{P}(x_i) + \sum_{i=1,2} \mathbbm{1}_{h_S(x_i) \neq h_2(x_i)}\mathbbm{1}_{x_i \neq x_1} \mathbb{P}(x_i) \\ & =  \mathbbm{1}_{h_S(x_2) \neq h_1(x_2)} 4\epsilon + \mathbbm{1}_{h_S(x_2) \neq h_2(x_2)} 4\epsilon\\ & = 4\epsilon,
\end{align}
where we used that $h_1(x_2) = 1 = - h_2(x_2)$ and that $h_S$ is independent of the underlying labelling function.

Since the above holds for any $S'\in\mathcal{S}$, it also holds in expectation, conditioned on $S'\in\mathcal{S}$:
\begin{equation}
\mathbb{E}_{S'\in\mathcal{S}}\left(\mathcal{R}(h_S, h_1) + \mathcal{R}(h_S, h_2)\right) \geq 4\epsilon.
\end{equation}
Therefore, $\mathbb{E}_{S'\in\mathcal{S}}\left(\mathcal{R}(h_S, h_i)\right) \geq 2\epsilon$ for at least one of $i = 1,2$. Take $f$ to be $h_1$, if $h_1$ satisfies the inequality, and $h_2$ otherwise. Conditioning on $\{\mathcal{R}(h_S, f) \geq \epsilon\}$ and using $\mathcal{R}(h_S, f) \leq \mathbb{P}_{\mathcal{D}}(x\neq x_1) = 4\epsilon$:
\begin{align}
2\epsilon \leq \mathbb{E}_{S'\in\mathcal{S}}\left(\mathcal{R}(h_S, f)\right) & = \mathbb{E}_{S'\in\mathcal{S}}\left(\mathcal{R}(h_S, f) | \mathcal{R}(h_S, f) \geq \epsilon\right) \mathbb{P}_{S'\in\mathcal{S}}\left(\mathcal{R}(h_S, f) \geq \epsilon\right) \\ & + \mathbb{E}_{S'\in\mathcal{S}}\left(\mathcal{R}(h_S, f) | \mathcal{R}(h_S, f) < \epsilon\right) \mathbb{P}_{S'\in\mathcal{S}}\left(\mathcal{R}(h_S, f) < \epsilon\right) \\ & \leq 4\epsilon \mathbb{P}_{S'\in\mathcal{S}}\left(\mathcal{R}(h_S, f) \geq \epsilon\right) + \epsilon \mathbb{P}_{S'\in\mathcal{S}}\left(\mathcal{R}(h_S, f) < \epsilon\right) \\ & = \epsilon + 3 \epsilon \mathbb{P}_{S'\in\mathcal{S}}\left(\mathcal{R}(h_S, f) \geq \epsilon\right).
\end{align}
Hence,
\begin{align}
\mathbb{P}_{S'\in\mathcal{S}}\left(\mathcal{R}(h_S, f) \geq \epsilon\right) \geq \frac{1}{3\epsilon}\left(2\epsilon - \epsilon\right) = \frac{1}{3}
\end{align}
Finally,
\begin{align}
\mathbb{P}_{S'} \left(\mathcal{R}(\mathcal{L}(\mathcal{A}^s(S'))) \geq \epsilon\right) \geq \mathbb{P}_{S'}\left(\mathcal{R}(h_S, f) \geq \epsilon\right) & \geq \mathbb{P}_{S'\in\mathcal{S}}\left(\mathcal{R}(h_S, f) \geq \epsilon\right) \mathbb{P}_{S'}\left(S'\in\mathcal{S}\right) > \frac{1}{3}\frac{3}{20} = \frac{1}{20}.
\end{align}

\end{proof}

\clearpage
\section{Proof of Theorem \ref{thm:no_free_lunch}}
\label{app:no_free_lunch_proof}
\begin{theorem}
\label{thm:no_free_lunch}
Let $\mathcal{H} \subset \{h: \mathcal{X}\rightarrow \mathcal{Y}\}$ 
be a hypothesis space, let $m$ and $N$ be any integers and let 
$\good$ be a fixed subset of $[N]$ of size $k \in \{1, \ldots, N-1\}$. Let $S' \in \left(\mathcal{X}\times\mathcal{Y}\right)^{N\times m}$ be drawn \iid from $\mathcal{D}$.
Then the following statements hold for any multi-source learner $\mathcal{L}$:
\begin{itemize}
\item[(a)] Suppose that $\mathcal{H}$ is non-trivial. 
Then there exists a distribution $\mathcal{D}$ on $\mathcal{X}$ 
with $\min_{h\in\mathcal{H}}\mathcal{R}(h)=0$, and 
a fixed-set adversary $\mathcal{A}$ with index set $G$, 
such that:
\begin{align}
\label{eqn:general_lower_bound_1-appendix}
\mathbb{P}_{S'} \Big(\mathcal{R}\big(\mathcal{L}(\mathcal{A}(S')\big) > \frac{\alpha}{8m}  \Big) > \frac{1}{20}.
\end{align}
\item[(b)] Suppose that $\mathcal{H}$ has VC dimension $d \geq 2$. 
Then there exists a distribution $\mathcal{D}$ on $\mathcal{X}\times \mathcal{Y}$ and 
a fixed-set adversary $\mathcal{A}$ with index set $G$, such that:
\begin{align}
\label{eqn:general_lower_bound_all-appendix}
\mathbb{P}_{S'} \Bigg(\mathcal{R}\big(\mathcal{L}(\mathcal{A}(S')\big) - \min_{h\in\mathcal{H}}\mathcal{R}(h) 
> \sqrt{\frac{d}{1280Nm}} + \frac{\alpha}{16m} \Bigg) > \frac{1}{64}. 
\end{align}
\end{itemize}
In both cases, $\alpha=\frac{N-k}{N}$ is the power of the adversary. 
\end{theorem}

To prove part (a), we use a similar technique as in the no-free-lunch results in \cite{bshouty2002pac} and in the classic no-free-lunch theorem, \eg Theorem 3.20 in \cite{mohri2018foundations}. An overview is as follows. Consider a distribution on $\mathcal{X}$ that has support only at two points - the common point $x_1$ and the rare point $x_2$. Take $\mathbb{P}(x_2) = \mathcal{O}(\frac{\alpha}{m})$. Then one can show that with constant probability the number of datasets that contain $x_2$ is at most $\alpha N$. We show that in this case there exists an algorithm for the strong adversary that will return the same unordered collection of datasets, regardless of the true label of $x_2$. Thus no learner can guess with probability greater than $1/2$ what the true label of $x_2$ was. 

Part (b) follows from part (a) and the standard no-free-lunch theorem for agnostic binary classification.

\begin{proof}
a) As in Theorem \ref{thm:lower_bound_single_source_learner}, we prove that there exists a distribution $\mathcal{D}$ on $\mathcal{X}$ and a labeling function $f\in\mathcal{H}$, such that the resulting joint distribution on $\mathcal{X}\times\mathcal{Y}$, defined by $x\sim \mathcal{D}$ and $y = f(x)$, satisfies the desired property.

Without loss of generality, let $\good = [1, 2, \ldots, k]$. Since $\mathcal{H}$ is non-trivial ($d\geq 2$), there exist $h_1, h_2 \in \mathcal{H}$ and $x_1, x_2 \in\mathcal{X}$, such that $h_1(x_1) = h_2(x_1)$, while $h_1(x_2) = 1$, but $h_2(x_2) = -1$. Consider the following distribution on $\mathcal{X}$:
\begin{equation}
\mathbb{P}_{\mathcal{D}}(x_1) = 1 - 4\epsilon \quad \text{and} \quad \mathbb{P}_{\mathcal{D}}(x_2) = 4\epsilon,
\end{equation}
where $\epsilon = \frac{\alpha}{8m}$. Assume that the points are labelled by a function $f\in\mathcal{H}$ (to be chosen later as either $h_1$ or $h_2$). Denote the initial uncorrupted collection of datasets by $S' = (S'_1, \ldots, S'_N)$, with $S'_i = \{(x'_{i,1}, f(x'_{i,1})), \ldots, (x'_{i,m}, f(x'_{i,m}))\}$ and $x'_{i,j}$ being  \iid samples from $\mathcal{D}$.

First we show that with constant probability the point $x_2$ is contained in no more than $\alpha N$ sources. Indeed, let $C_b$ be the number of sources that contain $x_2$ and let $C_p$ be the number of points (out of the $Nm$ in total) that are equal to $x_2$. Clearly $C_b \leq C_p$. Note that $C_p$ is a binomial random variable with probability of success $4\epsilon$ and number of trials $Nm$. Therefore, by the Chernoff bound:
\begin{equation}
\mathbb{P}_{S'}(C_p \geq \alpha N) = \mathbb{P}_{S'}(C_p \geq (1+1)4\epsilon Nm) \leq e^{-\frac{\alpha N}{6}} \leq e^{-1/6} < \frac{17}{20}
\end{equation}
and so:
\begin{equation}
\mathbb{P}_{S'}(C_b \leq \alpha N) \geq \mathbb{P}_{S'}(C_p \leq \alpha N) > \frac{3}{20}.
\end{equation}
Now consider the following policy for the adversary $\mathcal{A}^{s}: S' \rightarrow S$. Whenever $C_b \leq \alpha N$, let $M \subset \good$ be the list of indexes $i\in\good$, such that $S'_i$ contains $x_2$. Let $l = |M|$ and note that $l \leq C_b \leq \alpha N$. For any index $i\in [N]$ the adversary replaces $S'_i = \{x'_{i, 1}, f(x'_{i, 1}), \ldots, (x'_{i, m}, f(x'_{i,m}))\}$ with a dataset $S_i = \{(x_{i, 1}, y_{i, 1}) \ldots, (x_{i, m}, y_{i, m})\}$, such that:

\begin{align}
    (x_{i,j}, y_{i, j})= 
\begin{cases}
    (x'_{i,j}, f(x'_{i, j})),& \text{if } i\in\good = [1, 2, \ldots, k]\\
    (x_1, f(x_1)), & \text{if } i\in [k+1, \ldots, k+l] \text{ and } x'_{M[i-k], j} = x_1 \\
    (x_2, - f(x_2)), & \text{if } i\in [k+1, \ldots, k+l] \text{ and } x'_{M[i-k], j} = x_2 \\
    (x_1, f(x_1)), & \text{if } i\in [k+l+1, \ldots, N] \\
\end{cases}
\end{align}
Then the adversary returns $S = (S_1, \ldots, S_N)$. That is, the adversary keeps the datasets in $\good$ untouched, copies all of the datasets in $M$ into its own data, flipping the labels of the $x_2$'s, and, in case there are additional sources at its disposal, it fills them with (correctly labelled) $x_1$'s only.

Crucially, the resulting (unordered) collection is the same no matter if the original labelling function was $h_1$ or $h_2$. In particular, $\mathcal{L}(S)$ will be the same in both cases.

In the case when $C_b > \alpha N$, the adversary leaves the data unchanged, \ie $S = S'$.

Finally, we argue that under the event $C_b \leq \alpha N$ and the chosen adversary, the learner would incur high loss and show that this implies the result in (\ref{eqn:general_lower_bound_1}). Let $\mathcal{S}$ be the set of all datasets in $\left(\mathcal{X}\times\mathcal{Y}\right)^{N\times m}$, such that $C_b \leq \alpha N$ holds. We just showed that $\mathbb{P}_{S'}(S' \in\mathcal{S}) > \frac{3}{20}$ and that whenever $S' \in\mathcal{S}$, $\mathcal{L}(\mathcal{A}^s(S'))$ is independent of whether the original labelling function was $h_1$ or $h_2$.

Now the proof proceeds just as in Theorem \ref{thm:lower_bound_single_source_learner}. Consider a fixed set $S'\in\mathcal{S}$ and let $S = \mathcal{A}^s(S')$ and $h_S = \mathcal{L}(S)$. Denote by $\mathcal{R}(h_S, f) = \mathbb{P}_{\mathcal{D}}(h_S(x) \neq f(x) \cap x \neq x_1)$ and note that $\mathcal{R}(h_S, f) \leq \mathbb{P}_{\mathcal{D}}(h_S(x) \neq f(x)) = \mathcal{R}(\mathcal{L}(\mathcal{A}^s(S')))$. Notice that:

\begin{align}
\mathcal{R}(h_S, h_1) + \mathcal{R}(h_S, h_2) & = \sum_{i=1,2} \mathbbm{1}_{h_S(x_i) \neq h_1(x_i)}\mathbbm{1}_{x_i \neq x_1} \mathbb{P}(x_i) + \sum_{i=1,2} \mathbbm{1}_{h_S(x_i) \neq h_2(x_i)}\mathbbm{1}_{x_i \neq x_1} \mathbb{P}(x_i) \\ & =  \mathbbm{1}_{h_S(x_2) \neq h_1(x_2)} 4\epsilon + \mathbbm{1}_{h_S(x_2) \neq h_2(x_2)} 4\epsilon\\ & = 4\epsilon,
\end{align}
where we used that $h_1(x_2) = 1 = - h_2(x_2)$ and that $h_S$ is independent of the underlying labelling function.

Since the above holds for any $S'\in\mathcal{S}$, it also holds in expectation, conditioned on $S'\in\mathcal{S}$:
\begin{equation}
\mathbb{E}_{S'\in\mathcal{S}}\left(\mathcal{R}(h_S, h_1) + \mathcal{R}(h_S, h_2)\right) \geq 4\epsilon.
\end{equation}
Therefore, $\mathbb{E}_{S'\in\mathcal{S}}\left(\mathcal{R}(h_S, h_i)\right) \geq 2\epsilon$ for at least one of $i = 1,2$. Take $f$ to be $h_1$, if $h_1$ satisfies the inequality, and $h_2$ otherwise. Conditioning on $\{\mathcal{R}(h_S, f) \geq \epsilon\}$ and using $\mathcal{R}(h_S, f) \leq \mathbb{P}_{\mathcal{D}}(x\neq x_1) = 4\epsilon$:
\begin{align}
\label{eqn:general_lower_bound_1b}
2\epsilon \leq \mathbb{E}_{S'\in\mathcal{S}}\left(\mathcal{R}(h_S, f)\right) & = \mathbb{E}_{S'\in\mathcal{S}}\left(\mathcal{R}(h_S, f) | \mathcal{R}(h_S, f) \geq \epsilon\right) \mathbb{P}_{S'\in\mathcal{S}}\left(\mathcal{R}(h_S, f) \geq \epsilon\right) \\ & + \mathbb{E}_{S'\in\mathcal{S}}\left(\mathcal{R}(h_S, f) | \mathcal{R}(h_S, f) < \epsilon\right) \mathbb{P}_{S'\in\mathcal{S}}\left(\mathcal{R}(h_S, f) < \epsilon\right) \\ & \leq 4\epsilon \mathbb{P}_{S'\in\mathcal{S}}\left(\mathcal{R}(h_S, f) \geq \epsilon\right) + \epsilon \mathbb{P}_{S'\in\mathcal{S}}\left(\mathcal{R}(h_S, f) < \epsilon\right) \\ & = \epsilon + 3 \epsilon \mathbb{P}_{S'\in\mathcal{S}}\left(\mathcal{R}(h_S, f) \geq \epsilon\right).
\end{align}
Hence,
\begin{align}
\mathbb{P}_{S'\in\mathcal{S}}\left(\mathcal{R}(h_S, f) \geq \epsilon\right) \geq \frac{1}{3\epsilon}\left(2\epsilon - \epsilon\right) = \frac{1}{3}
\end{align}
Finally,
\begin{align}
\mathbb{P}_{S'} \left(\mathcal{R}(\mathcal{L}(\mathcal{A}^s(S'))) \geq \epsilon\right) & \geq \mathbb{P}_{S'}\left(\mathcal{R}(h_S, f) \geq \epsilon\right) \\ & \geq \mathbb{P}_{S'\in\mathcal{S}}\left(\mathcal{R}(h_S, f) \geq \epsilon\right) \mathbb{P}_{S'}\left(S'\in\mathcal{S}\right) \\ & > \frac{1}{3}\frac{3}{20} = \frac{1}{20}.
\end{align}

b) First we argue that there exists a distribution $\mathcal{D}_1$ on $\mathcal{X}\times \mathcal{Y}$ and a fixed-set adversary $\mathcal{A}^s_1$, such that:
\begin{align}
\label{eqn:general_lower_bound_2}
\mathbb{P}_{S' \sim \mathcal{D}_1} \left(\mathcal{R}(\mathcal{L}(\mathcal{A}_1^s(S'))) - \min_{h\in\mathcal{H}}\mathcal{R}(h) > \sqrt{\frac{d}{320Nm}} \right) > \frac{1}{64}.
\end{align}
This follows directly from the classic no-free-lunch theorem for binary classifiers in the unrealizable case. Indeed, applying Theorem 3.23 in \cite{mohri2018foundations} and setting the adversary to be the identity mapping gives the result.

Now, since any hypothesis space with VC dimension $d\geq 2$ is non-trivial, we also know from a) that there exists an adversary $\mathcal{A}_2^{s}$ and a distribution $\mathcal{D}_2$ on $\mathcal{X}\times\mathcal{Y}$, such that:
\begin{align}
\mathbb{P}_{S' \sim \mathcal{D}_2} \left(\mathcal{R}(\mathcal{L}(\mathcal{A}_2^s(S'))) -  \min_{h\in\mathcal{H}}\mathcal{R}(h) > \frac{\alpha}{8m}  \right) > \frac{1}{20}.
\end{align}
Fix any set of values for $N, m, d, k$. Then at least one of the pairs $\left(\mathcal{A}_1^{s}, \mathcal{D}_1\right)$ and $\left(\mathcal{A}_2^{s}, \mathcal{D}_2\right)$ satisfies:
\begin{align}
\mathbb{P}_{S'} \left(\mathcal{R}(\mathcal{L}(\mathcal{A}^s(S'))) -  \min_{h\in\mathcal{H}}\mathcal{R}(h) > \sqrt{\frac{d}{1280Nm}} + \frac{\alpha}{16m} \right) & \geq \mathbb{P}_{S'} \left(\mathcal{R}(\mathcal{L}(\mathcal{A}^s(S'))) > 2\max \{\sqrt{\frac{d}{1280Nm}}, \frac{\alpha}{16m}\} \right) \\ & = \mathbb{P}_{S'} \left(\mathcal{R}(\mathcal{L}(\mathcal{A}^s(S'))) > \max \{\sqrt{\frac{d}{320Nm}}, \frac{\alpha}{8m}\} \right) \\ & > \frac{1}{64}.
\end{align}

\end{proof}

\end{document}